%% file: main.tex
\documentclass{article}

\usepackage{appendix}
\usepackage{amsmath, amsfonts, amssymb, amsthm}
\usepackage[margin=1in]{geometry}
\usepackage{graphicx}
\usepackage{hyperref}
\hypersetup{colorlinks, citecolor=blue}
\usepackage{color}
\usepackage{cleveref}
\usepackage{bm}

\let \hat \widehat

\newcommand{\cL}{\mathcal L}
\newcommand{\cN}{\mathcal N}
\newcommand{\bE}{\mathbb{E}}
\newcommand{\bx}{\mathbf{x}}
\newcommand{\bX}{\mathbf{X}}
\newcommand{\by}{\mathbf{y}}
\newcommand{\bz}{\mathbf{z}}
\newcommand{\cX}{\mathcal X}
\newcommand{\bR}{\mathbb{R}}
\newcommand{\bW}{\mathbf{W}}

\newcommand{\rd}{\mathrm{d}}
\newcommand{\KL}{\mathrm{KL}}
\newcommand{\TV}{\mathrm{TV}}

\newcommand{\distri}{p_{\mathrm{data}}}

\newcommand{\score}{\mathrm{score}}

\newtheorem{Theorem}{Theorem}[section]
\newtheorem{Assumption}{Assumption}[section]
\newtheorem{Proposition}{Proposition}[section]
\newtheorem{Corollary}{Corollary}[section]
\newtheorem{Lemma}{Lemma}[section]
\newtheorem{Remark}{Remark}[section]

\title{\bf Provable Statistical Rates for Consistency Diffusion Models}
\author{Zehao Dou\footnote{Department of Statistics and Data Science, Yale University} \and Minshuo Chen\footnote{Electrical and Computer Engineering, Princeton University} \and Mengdi Wang\(^\dagger\) \and Zhuoran Yang\(^*\)}
\date{}

\begin{document}
\maketitle
\begin{abstract}
\input{abstract}
\end{abstract}

\allowdisplaybreaks[4]
\section{Introduction}
\input{intro}

\section{Diffusion Model Preliminary}
\label{sec: back}
\input{background}

\section{Statistical Rates of Consistency Models}
\label{sec: distillation}
\input{main_theorem1}

\section{Proof Sketch for Consistency Distillation}
\label{sec: sketch}
In this section, we provide the proof sketch for the main theorems proposed in the previous part. First, we propose an overview of the entire proof sketch.
\subsection{Technical Overview}
\input{proof-overview}

\subsection{Approximation Error for Score Estimation}
\input{approx-1}

\subsection{Upper Bound the Consistency Loss}
\input{decomposition-consistency}

\subsection{Proof of Main Theorem 1}
\input{proof-thm1}

\section{Conclusion}
\input{conclusion}

\bibliography{reference}
\bibliographystyle{alpha}

\newpage
\appendix
\onecolumn
\input{appendix}

\end{document}

%% file: abstract.tex
Diffusion models have revolutionized various application domains, including computer vision and audio generation. Despite the state-of-the-art performance, diffusion models are known for their slow sample generation due to the extensive number of steps involved. In response, consistency models have been developed to merge multiple steps in the sampling process, thereby significantly boosting the speed of sample generation without compromising quality. This paper contributes towards the first statistical theory for consistency models, formulating their training as a distribution discrepancy minimization problem. Our analysis yields statistical estimation rates based on the Wasserstein distance for consistency models, matching those of vanilla diffusion models. Additionally, our results encompass the training of consistency models through both distillation and isolation methods, demystifying their underlying advantage.

%% file: intro.tex
Diffusion models have reached state-of-the-art performance in cross-domain applications, including computer vision \cite{song2019generative, dathathri2019plug, ho2020denoising, song2020score}, audio generation \cite{kong2020diffwave, chen2020wavegrad}, language generation \cite{li2022diffusion, yu2022latent, lovelace2022latent}, reinforcement learning and control \cite{pearce2023imitating, chi2023diffusion, hansen2023idql, reuss2023goal}, as well as computational biology \cite{lee2022proteinsgm, luo2022antigen, gruver2023protein}. These break-through performances are enabled by the unique design in the diffusion models. Specifically, diffusion models utilize the forward and backward processes to generate new samples. In the forward process, a clean data point is progressively contaminated by random noise, while the backward process attempts to remove the noise iteratively (typically taking 500 to 1000 steps \cite{song2019generative}) with the help of a specific type of neural network known as a score neural network.

~\\
Due to the enormous size of the score neural network, e.g., the smallest stable diffusion model uses a network of more than 890M parameters \cite{rombach2022high}, the sample generation speed of diffusion models is limited \cite{song2023consistency}, compared to Generative Adversarial Networks \cite{goodfellow2020generative} and AutoEncoders \cite{kingma2019introduction}. To overcome this shortcoming, there are extensive methodological studies aiming to accelerate diffusion models. Notable methods include using stride in sampling to reduce the number of backward steps \cite{nichol2021improved, song2020improved, lu2022dpm}, changing the backward process to a deterministic probabilistic flow \cite{song2020denoising, karras2022elucidating, zhang2022gddim}, and utilizing pretrained variational autoencoders to reduce the data dimensionality before applying diffusion models \cite{rombach2022high}. These methods lead to sampling speed acceleration, but may compromise the quality of the generated samples.

~\\
More recently, consistency models \cite{song2023consistency} achieve a significant sampling speed boost, while maintaining the high quality in generated samples. Roughly speaking, consistency models merge a large number of consecutive steps in the original backward process by additionally training a consistency network via distillation or isolation. The distillation method requires a pretrained diffusion model, yet isolation lifts this requirement. In either ways, it suffices to deploy the consistency model for very few times or even a single time to generate a new sample.

~\\
Despite the empirical success, theoretical underpinnings of consistency models are limited. In particular, the following question is largely open:
\begin{center}
\it What is the statistical error rate of consistency models for estimating the data distribution? How does it compare to the vanilla diffusion models?
\end{center}
In this paper, we provide the first theoretical study towards a positive answer to the preceding question. Specifically, we consider both the distillation and isolation methods and establish statistical estimation rate of consistency models in terms of the Wasserstein distance. We summarize our contributions as follows:
\begin{itemize}
    \item We formulate the training of consistency models as a Wasserstein distance minimization problem. This formulation is the first principled objective of consistency models, encompassing the practical consistency models' training proposed in \cite{song2023consistency}.

    \item We establish statistical distribution estimation guarantees of consistency models trained under the distillation method. We demonstrate in Theorem~\ref{thm:distillation} that the distribution estimation error is dominated by the score estimation error, showing that consistency models preserve the distribution estimation ability of vanilla diffusion models, but allow efficient sample generation.

    \item We extend our study to the isolation method, establishing analogous statistical estimation result. An $\widetilde{O}(n^{-1/d})$ statistical error rate is obtained in Theorem~\ref{thm:isolation} without any pretraining on the score function.
\end{itemize} 
These results are the first attempt to demystify consistency models from a statistical estimation perspective.

\subsection{Related Work}
Our work is related to the recent sampling theory of consistency models \cite{lyu2023convergence}, where they assume the score function as well as a multi-step backward process sampler have been accurately estimated. Our analysis, however, does not require such assumptions. In fact, we provide sample complexity bounds of ensuring these estimation errors being small. Apart from \cite{lyu2023convergence}, recent theoretical advances in diffusion models can be roughly categorized into sampling and statistical theories.

\paragraph{Sampling Theory of Diffusion Models} This line of works show that the distribution generated by a diffusion model is close to the data distribution, as long as the score function is assumed to be accurately estimated. Specifically, \cite{de2021diffusion, albergo2023stochastic} study sampling from diffusion Schr\"{o}dinger bridges with $L_\infty$ accurate score functions. Concrete sampling distribution error bounds of diffusion models are provided in \cite{block2020generative, lee2022convergencea, chen2022sampling, lee2022convergenceb} under different settings, yet they all assume access to $L_2$ accurate score functions. \cite{lee2022convergencea} require the data distribution satisfying a log-Sobolev inequality. Concurrent works \cite{chen2022sampling} and \cite{lee2022convergenceb} relax the log-Sobolev assumption to only having bounded moments conditions.

It is worth mentioning that \cite{lee2022convergenceb} allow the error of the score function to be time-dependent. Recently, \cite{chen2023restoration, chen2023probability, benton2023linear} largely enrich the study of sampling theory using diffusion models. Specifically, novel analyses based on Taylor expansions of the discretized backward process \cite{li2023diffusion} or localization method \cite{benton2023linear} are developed. Further, \cite{chen2023restoration, chen2023probability} extend to broad backward sampling methods. Besides Euclidean data, \cite{de2022convergence} made the first attempt to analyze diffusion models for learning low-dimensional manifold data. Moreover, \cite{montanari2023posterior} consider using diffusion processes to sample from noisy observations of symmetric spiked models and \cite{el2023sampling} study polynomial-time algorithms for sampling from Gibbs distributions based on diffusion processes.

\paragraph{Statistical Theory of Diffusion Models}
Distribution estimation bounds of diffusion models are first explored in \cite{song2020sliced} and \cite{liu2022let} from an asymptotic statistics point of view. These results do not provide an explicit sample complexity bound. Later, \cite{oko2023diffusion} and \cite{chen2023score} establish sample complexity bounds of diffusion models for both Euclidean data and low-dimensional subspace data. More recently, \cite{yuan2023reward} study the distribution estimation of conditional diffusion models with scalar reward guidance. \cite{mei2023deep} investigate statistical properties of diffusion models for learning high-dimensional graphical models.

\smallskip
\noindent \textbf{Notation}: 
For a mapping $F: \mathbb R^D \rightarrow \mathbb R^d$ and a distribution $\mathcal D$ supported on $\mathbb R^D$, $F_\sharp \mathcal D$ stands for the push forward distribution, which means $\mathrm{Law}(F(\bx))$ where $\bx\sim \mathcal D$. For brevity, we denote $c \mathcal D := f_\sharp \mathcal D$ where $f: \bx\mapsto c\bx$ is a scaling function. For two distributions $\mathcal D_1, \mathcal D_2$ supported on $\bR^d$, denote $\mathcal D_1 \star \mathcal D_2$ as their convolution, which stands for $\mathrm{Law}(\bx+\by)$ where $\bx\sim \mathcal D_1$ and $\by\sim \mathcal D_2$. The given dataset is $\{\bx^j\}_{j\in [n]}$, which is assumed to be i.i.d sampled from $\distri$, our target distribution. The empirical distribution is denoted as $\hat{\distri} = \frac1n\sum_j \delta_{\bx^j}$. Here, $\delta_\bx$ stands for the Dirac delta distribution at point $\bx$. 

%% file: background.tex
We adopt a continuous time description of diffusion models, which provide rich interpretations. In practice, a proper discretization is applied accordingly. Diffusion models consist of two coupled processes. In the forward process, we gradually add noise to data following a stochastic differential equation:
\begin{equation}
\label{eqn-sde}
\rd\bx_t = \mu(\bx_t, t) \rd t + \sigma(t) \rd \bW_t~~~\text{for}~t\in[0,T].
\end{equation}
Here $\bW_t(\cdot)$ is the standard Brownian motion, $\mu(\bx_t, t)$ and $\sigma(t)$ are the drift term and the diffusion term respectively, and $T$ is a terminal time. The forward process \eqref{eqn-sde} starts from $\bx_0\sim \distri$, the distribution of data. At each time $t\in [0,T]$, we denote $\bx_t\sim p_t$ as the marginal distribution of the forward process.

As shown in \cite{anderson1982reverse}, the forward process enjoys a time reversal, which is termed as the backward process:
\begin{equation}
\label{eqn-sde-backward}
\rd\bx_t = \Big[\mu(\bx_t, t) - \sigma(t)^2 \nabla_{\bx_t} \log p_t(\bx_t)\Big] \rd t + \sigma(t) \rd\overline{\bW}_t. 
\end{equation}
Here $\overline{W}_t(\cdot)$ is a standard Brownian motion with time flowing backward from $T$ to 0 and $\nabla_{\bx_t} \log p_t(\bx_t)$ is the score function. In practice, we use a score neural network to estimate the unknown score function via denoising score matching \cite{song2019generative, vincent2011connection}. It is worth mentioning that \eqref{eqn-sde-backward} is not the only backward process whose solution trajectories match the distribution of the forward process. We present the following example.

A commonly used specialization of \eqref{eqn-sde} is the variance preserving SDE (VP-SDE, \cite{dhariwal2021diffusion}), i.e., 
\begin{equation}
\label{eqn:vp-sde}
\rd\bx_t = -\frac{\beta(t)}{2} \bx_t \rd t + \sqrt{\beta(t)} \rd\bW_t.
\end{equation}
Here $\beta(t) > 0$ is the noise schedule, which is usually chosen as a linear function over $t$. Under VP-SDE, we have that the transition kernel $p(\bx_t\mid \bx_0)$ is Gaussian satisfying
\begin{equation}
\label{eqn:Gaussian}
p(\bx_t\mid \bx_0) = \mathcal N(\bx_t\mid m(t)\bx_0, \sigma(t)^2 \bm I),
\end{equation}
where
\[m(t) = \exp\left(-\frac12\int_0^t \beta(s)\rd s\right)~~~\text{and}~~~\sigma(t)^2 = 1 - m(t)^2. \]
At the terminal time $T$, the marginal distribution $p_T := p(\bx_T)$ is approximately a standard Gaussian distribution. The corresponding backward process to \eqref{eqn:vp-sde} is
\begin{equation}
\rd\bx_t = \left[-\frac{\beta(t)}{2} \bx_t - \beta(t) \nabla_{\bx_t} \log p_t(\bx_t)\right] \rd t + \sqrt{\beta(t)} \rd\overline{\bW}_t \nonumber
\end{equation}
Interestingly, \eqref{eqn:vp-sde} also assumes a probability ODE flow as a backward process:  
\begin{equation}
\label{eqn:vp-sde-back-ode}
\rd\bx_t = \left[-\frac{\beta(t)}{2} \bx_t - \frac{\beta(t)}{2} \nabla_{\bx_t} \log p_t(\bx_t)\right] \rd t.
\end{equation}
As can be seen, the transition in \eqref{eqn:vp-sde-back-ode} is deterministic.

\section{Consistency Models Minimize Discrepancy}\label{sec:consistency_train}
Consistency models merge multiple backward steps in the vanilla diffusion models to expedite the sampling. As proposed in \cite{song2023consistency}, the training of consistency models utilizes either distillation or isolation. Unfortunately, only iterative algorithms are derived in \cite{song2023consistency}, making the training objective elusive. In this section, we formulate the training of consistency models as a Wasserstein distance minimization problem, which encodes the original derivation in \cite{song2023consistency}, but also enables broad modifications.

To motivate the consistency models, we consider the probabilistic ODE \eqref{eqn:vp-sde-back-ode} as our backward process. Consistency models seek a mapping $f_\theta(\bx, t)$ that identifies a solution trajectory in the backward ODE to a single point. In particular, we define
\begin{equation}
\label{eqn: sectioned}
f_{\theta}(\bx, t) = \begin{cases}
\bx~~~~~~~~~ & \mbox{$t=\varepsilon$} \\
F_{\theta}(\bx, t) ~~~~~&\mbox{$t \in (\varepsilon, T]$}
\end{cases},
\end{equation}
where $F_{\theta}(\cdot, \cdot) : \bR^d \times [\epsilon, T] \mapsto \bR^d$ is a free-form deep neural network with parameter $\theta$ and $\varepsilon$ is an early-stopping time to prevent instability \cite{song2020improved}. The neural network $F_\theta(\bx, t)$ should satisfy a time-invariant property with respect to the solution trajectories in the ODE \eqref{eqn:vp-sde-back-ode}. Specifically, for any two time points $t_1 \neq t_2 \in [\varepsilon, T]$, we denote the contemporary generated samples as $\bx_{t_1}$ and $\bx_{t_2}$. Then in the ideal case, it holds that $F_\theta(\bx_{t_1}, t_1) = F_\theta(\bx_{t_2}, t_2) = \bx_{\varepsilon}$. In other words, $F_{\theta}$ attempts to identify an ODE trajectory to its end point, which is the generated data point.
\begin{figure}
    \centering
    \includegraphics[width=0.75\linewidth]{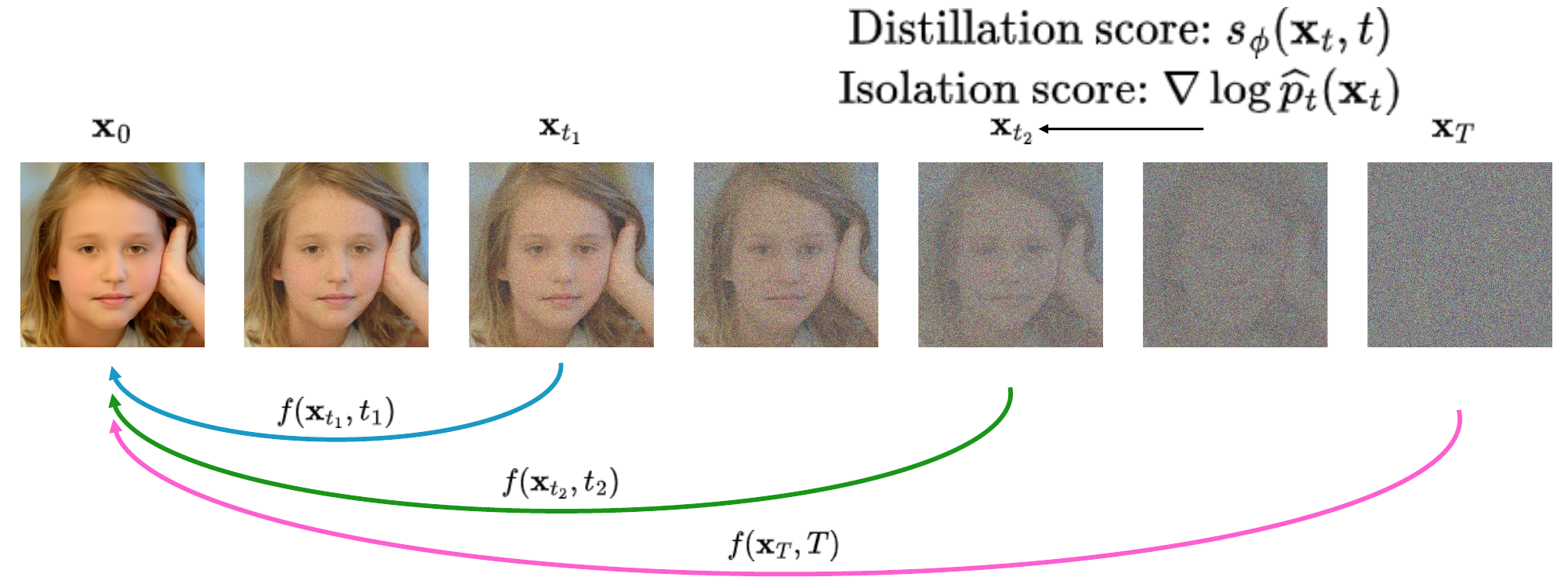}
    \caption{Illustration of Consistency Models: At each time step $t$, the consistency model $f(\cdot, t)$ will map $\bx_t$ to $\bx_0$ along the trajectory of probability flow ODE. We also demonstrate the score function applied at time $t$ in both distillation training and isolation training.}
    \label{fig:1}
\end{figure}

\paragraph{Training of Consistency Models} The training of consistency models leverage the time-invariance of $f_\theta(\bx, t)$. We discretize the time interval $[\varepsilon, T]$ into $N$ uniform sub-intervals, with breaking points $\varepsilon = t_0 < t_1 < \ldots < t_N = T$. We denote $t_k = t_0 + k \Delta t$ where $\Delta t = \frac{T-\varepsilon}{N}$ is the length of each sub-interval. We also denote $\{\tau_k\}_{k\in [N']}$ as a subset of time steps such that $\tau_k := t_{kM}$ where $N = N'M$ with $\tau_0 = t_0 = \varepsilon$ and $\tau_{N'} = t_N = T$. Corresponding to the exposure of the time-invariance property of $f_\theta$, consistency models aim to enforce
\[f_\theta(\cdot, \tau_k)_\sharp \bX_{\tau_k} \overset{\rm law}{=} f_\theta(\cdot, \tau_{k-1})_\sharp \bX_{\tau_{k-1}} \overset{\rm law}{=} \bX_\varepsilon ~~~\forall k\in [N'].\]
To this end, we define the following Wasserstein distance-based consistency loss for training $f_{\theta}$:
\begin{equation}
\label{eqn:consistency-loss-1}
\sum_{k=1}^{N'} W_1\left(f_{\theta}(\cdot, \tau_k)_\sharp \bX_{\tau_k}, f_{\theta}(\cdot, \tau_{k-1})_\sharp \bX_{\tau_{k-1}}\right). 
\end{equation}
Here $\bX_t = \mathrm{Law}(\bx_t) = m(t) \distri \star \mathcal N(0, \sigma(t)^2 \bm I)$ for $\forall t\in [\varepsilon, T]$ by \cref{eqn:Gaussian}. We remark that \eqref{eqn:consistency-loss-1} accommodates to both deterministic and stochastic backward processes by measuring the distribution discrepancy, while our discussion focuses on deterministic ODEs.

Notice that the Wasserstein distance in \eqref{eqn:consistency-loss-1} is not tractable, since we have no access to the target distribution $\distri$, let alone $\bX_t$. Therefore, we replace it by the empirical counterpart $\hat{\distri}=\frac1n\sum_j \delta_{\bx^j}$ as well as $\cX_t := m(t) \hat{\distri} \star \mathcal N(0, \sigma^2(t))$, the empirical version of $\bX_t$. We cast \eqref{eqn:consistency-loss-1} into 
\begin{equation}
\label{eqn:consistency-loss-2}
\sum_{k=1}^{N'} W_1\left(f_{\theta}(\cdot, \tau_k)_\sharp \cX_{\tau_k}, f_{\theta}(\cdot, \tau_{k-1})_\sharp \cX_{\tau_{k-1}}\right). 
\end{equation}
In practice, there are two different approaches, named distillation and isolation, to determine the corresponding sample from $\bX_{\tau_{k-1}}$ given a sample $\bx_{\tau_k}\sim \bX_{\tau_k}$. Both of them pushes $\bx_{\tau_k}$ along the backward probability flow (\ref{eqn:vp-sde-back-ode}) by ODE update, but consistency distillation relies on a pretrained plug-in score estimator $s_\phi(\bx, t)$ while consistency isolation does not require any pretrained models. In this work, we study both the consistency distillation and isolation in \cref{sec: distillation} and provide a statistical rate of $W_1\left(f_{\hat{\theta}}(\cdot, T)_\sharp \mathcal N(0, \bm I), \distri\right)$ for the learned consistency model $f_{\hat{\theta}}(\cdot, \cdot)$. 
\paragraph{Distillation Method}
Given a time step $\tau_k$ as well as $\bx_{\tau_k}\sim \bX_{\tau_k}$, we obtain a corresponding sample $\bx_{\tau_{k-1}}$ by running $M$ discretization steps of probability ODE \eqref{eqn:vp-sde-back-ode} solver starting from $\bx_{\tau_k}$. For a one-step update, we denote
\begin{equation}
\label{eqn:1}
\hat{\bx}_{t_{k-1}}^\phi = \bx_{t_k}-\Delta t\cdot \Phi(\bx_{t_k}, t_k; \phi) := G(\bx_{t_k}, t_k;\phi).
\end{equation}
Here, $\Phi(\cdot, \cdot; \phi)$ is the update function of numerical ODE. In our variance preserving framework (\ref{eqn:vp-sde-back-ode}), we have:
\begin{equation}
\label{eqn:2}
\Phi(\bx_{t_k}, t_k; \phi) = -\frac{\beta(t_k)}{2}\bx_{t_k}-\frac{\beta(t_k)}{2}\cdot s_{\phi}(\bx_{t_k}, t_k). 
\end{equation}
After applying $M$ consecutive updates, we obtain $\hat{\bx}_{\tau_{k-1}}^{\phi, M}$ from $\bx_{\tau_k}\sim \bX_{\tau_k}$, which is defined as
\[\by_M := \bx_{\tau_k} = \bx_{t_{kM}}, ~\by_{j-1} := G\left(\by_j, t_{(k-1)M+j}; \phi\right)\]
for $j\in[M]$, and eventually $\hat{\bx}_{\tau_{k-1}}^{\phi, M}:= \by_0$. Here, the update function $G(\cdot, \cdot;\phi)$ is the same as \cref{eqn:1}. For simplicity, it is equivalent to express it as
\begin{equation*}
\hat{\bx}_{\tau_{k-1}}^{\phi, M} = G_{(M)}(\bx_{\tau_k}, \tau_k; 
\phi):= G(\cdot, t_{(k-1)M+1}; \phi)\circ \ldots \circ  G(\cdot, t_{kM}; \phi)(\bx_{\tau_k}). 
\end{equation*}

In this way, we can approximate distribution $\bX_{\tau_{k-1}}$ with $G_{(m)}(\cdot, \tau_k;\phi)_\sharp \bX_{\tau_k}$, whose error only comes from the discretization loss of ODE solver as well as the score estimation loss of $s_\phi(\cdot, t)$. Now, we have the training objective of consistency models as follows:
\begin{equation}
\label{eqn:consistency}
\cL_{\mathrm{CD}}^N (\theta;\phi) = \sum_{k=1}^{N'} W_1\left(f_{\theta}(\cdot, \tau_k)_\sharp \cX_{\tau_k}, f_{\theta}(\cdot , \tau_{k-1})_\sharp \hat{\cX}_{\tau_{k-1}}^{\phi,M}\right).
\end{equation}
Here, $\hat{\cX}_{\tau_{k-1}}^{\phi,M} = G_{(M)}(\cdot, \tau_k;\phi)_\sharp \cX_{\tau_k}$ is the underlying distribution of $\hat{\bx}_{\tau_{k-1}}^{\phi, M} = G_{(M)}(\bx_{\tau_k}, \tau_k;\phi)$ where $\bx_{\tau_k}\sim \cX_{\tau_k}$. Our consistency model $f_{\hat{\theta}}$ is optimized over function class $\mathrm{Lip}(R)$, with regard to the optimization problem:
\begin{equation}
\label{eqn:consistency-objective}
\hat{\theta} = \arg\min_{\theta:~f_{\theta}\in \mathrm{Lip}(R)} \cL_{\mathrm{CD}}^N (\theta;\phi). 
\end{equation}
Here, $\mathrm{Lip}(R)$ denotes the set of functions $f(\bx, t)$ such that $f(\cdot, t)$ is Lipschitz-$R$ continuous over $\bx$ at any given time step $t\in[\varepsilon, T]$ with boundary condition $f(\cdot,\varepsilon) = \mathrm{id}$. 

\begin{figure}
    \centering
    \includegraphics[width=0.75\linewidth]{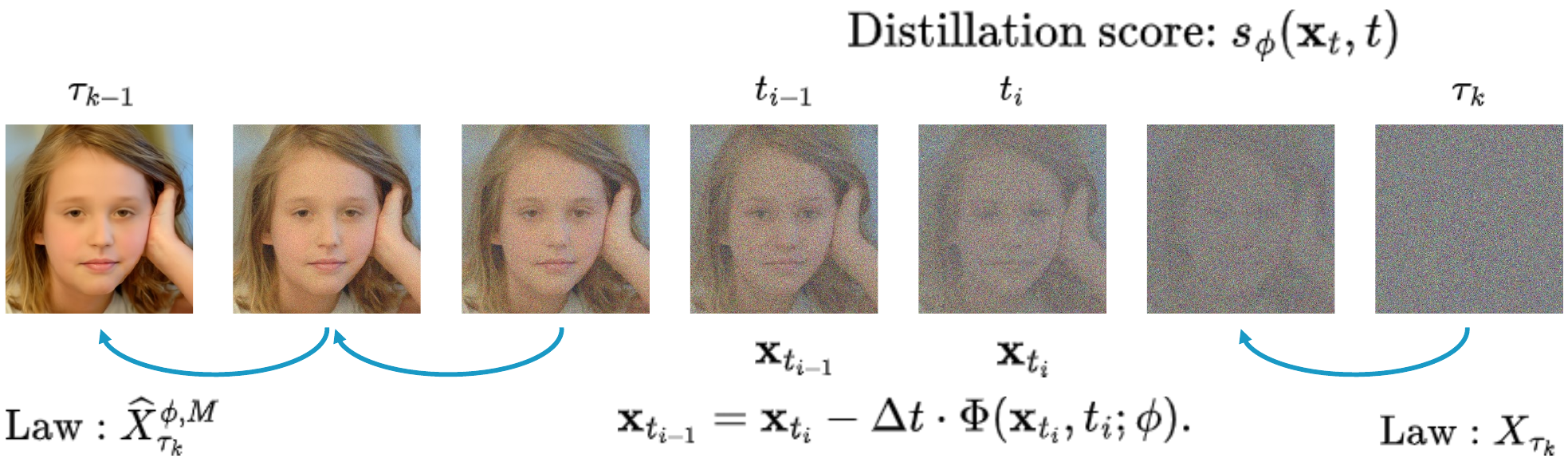}
    \caption{Illustration of $\widehat{X}_{\tau_k}^{\phi, M}$: When starting from distribution $X_{\tau_k}$ at time $\tau_k$ and following the discrete distillation-based backward process, it ends at $\tau_{k-1}$ with underlying law $\widehat{X}_{\tau_k}^{\phi, M}$. }
    \label{fig:2}
\end{figure}

\paragraph{Isolation Method}
Besides training in distillation, consistency models can also be trained without a pre-learned score estimator. Instead of using score model $s_\phi(\bx_t, t)$ to approximate the true score function $\nabla \log p_t(\bx_t)$, we can also use the following Tweedie's formula
\[\nabla\log p_t(\bx_t) = -\bE\left[\frac{\bx_t - m(t)\bx_0}{\sigma(t)^2}\bigg| \bx_t\right]\]
where $\bx_0\sim \distri$ and $p(\bx_t\mid \bx_0) = \mathcal N(\bx_0, \sigma(t)^2\bm I)$. Since $\distri$ is intractable, we make an unbiased approximation as follows:
\begin{equation}
\label{eqn:isolation-pre}
\nabla\log p_t(\bx_t) \approx -\bE_{\bx_0\sim \hat{\distri}}\left[\frac{\bx_t - m(t)\bx_0}{\sigma(t)^2}\bigg| \bx_t\right].
\end{equation}
\begin{Lemma}
\label{lemma:isolation-pre}
For the approximator above, it exactly equals to the score function of distribution $\cX_t$, i.e.
\[-\bE_{\bx_0\sim \hat{\distri}}\left[\frac{\bx_t - m(t)\bx_0}{\sigma(t)^2}\bigg| \bx_t\right] = \nabla\log \hat{p}_t(\bx_t).\]
Here, $\hat{p}_t(\cdot)$ is the density of $\cX_t = m(t)\hat{\distri}\star \mathcal N(0,\sigma(t)^2\bm I)$, which is a mixture of Gaussian. Therefore, it has explicit formulation and needs no additional training. 
\end{Lemma}
\begin{proof}
Detailed proof is left in Appendix \S \ref{sec:A.1}.
\end{proof}
\cref{lemma:isolation-pre} concludes that, taking a backward ODE step in the isolation setting is equivalent to moving along the following empirical backward diffusion ODE. 
\begin{equation}
\label{eqn:backward-total}
\begin{aligned}
&\text{Forward:~}\rd\bx_t = -\frac{\beta(t)}{2} \bx_t \rd t + \sqrt{\beta(t)} \rd\bW_t, ~\bx_0\sim \hat{\distri}. \\
&\text{Backward:~}\rd\bx_t = \left[-\frac{\beta(t)}{2} \bx_t - \frac{\beta(t)}{2} \nabla_{\bx_t} \log \hat{p}_t(\bx_t)\right] \rd t.
\end{aligned}
\end{equation}
Its only differences with the diffusion model introduced in distillation training is that $\bx_0\sim \hat{\distri}$ instead of $\distri$, and the empirical score function $\nabla \log \hat{p}_t(\cdot)$ is applied instead of true score $\nabla \log p_t(\cdot)$. Under this forward SDE, it's obvious that $\mathrm{Law}(\bx_t) = \cX_t$. In this case, a one-step update of the backward probability ODE at $\bx_{t_k}\sim \cX_{t_k}$ accurately links $\cX_{t_k}$ to $\cX_{t_{k-1}}$. Therefore, the isolation training objective of consistency models is as follows:
\begin{equation}
\label{eqn:consistency-2}
\cL_{\mathrm{CT}}^N (\theta) = \sum_{k=1}^{N'} W_1\left(f_{\theta}(\cdot, \tau_k)_\sharp \cX_{\tau_k}, f_{\theta}(\cdot , \tau_{k-1})_\sharp \cX_{\tau_{k-1}}\right).
\end{equation}
Similarly, our consistency model $f_{\hat{\theta}}$ is optimized with regard to the optimization problem:
\begin{equation}
\label{eqn:isolation-objective}
\hat{\theta} = \arg\min_{\theta:~f_{\theta}\in \mathrm{Lip}(R)} \cL_{\mathrm{CT}}^N (\theta). 
\end{equation}
Notice that, there is no parameter $\phi$ in the objective since the isolation training does not need the pre-trained score model $s_\phi(\cdot, \cdot)$. 

\paragraph{Connection to Original Consistency Model Training in \cite{song2023consistency}}
For the practical training of consistency models, \cite{song2023consistency} proposed the following sample-based consistency loss:
\begin{equation*}
\cL(\theta, \theta^-;\phi) = \mathbb{E}\left[\lambda(t_k)\cdot d\left(f_{\theta}(\bx_{t_k}, t_k), f_{\theta^-}(\hat{\bx}_{t_{k-1}}^\phi , t_{k-1})\right)\right].
\end{equation*}
Here, the expectation is taken over $k\sim \mathrm{Unif}[1,N]$ and $\bx_t \sim \cX_t$ for $\forall t\in [0,T]$. $\theta^-$ is the running average of the past values of $\theta$ in previous iterations during the optimization, and $d(\cdot, \cdot)$ is a metric function over the sample space. Besides, $\lambda(\cdot)$ is a positive weighting function over time and $\hat{\bx}_{t_{k-1}}^\phi$ is obtained by making a discretization step through backward probability flow (\cref{eqn:1}) from $\bx_{t_k}$. In comparison, we make the following minor modifications to ease the theoretical analysis on the statistical rate of consistency models.

~\\
We let $\lambda(\cdot)\equiv 1$, which is applied both practically and theoretically \cite{lyu2023convergence}. A simplification of $\theta^- = \theta$ is made since the optimization techniques while learning consistency models is not what we consider from the statistical point of view. We also extend the one-step ODE solver to multi-step ODE solver, which pushes $\bx_{\tau_k}$ back to $\hat{\bx}_{\tau_{k-1}}^{\phi,M}$. 

~\\
Another main difference is that we use Wasserstein-1 metric $W_1(\cdot, \cdot)$ over distribution space instead of the sample-based metric $d(\cdot, \cdot)$ as the training objective of consistency models. Our ultimate goal is to upper bound the distance between $\distri$ and $f_{\hat{\theta}}(\cdot, T)_\sharp \mathcal N(0,\bm I)$, which makes the distribution-based metric sufficient for our analysis. 

~\\
In specific, consistency models aim to learn a direct transformation $f_\theta$ that matches the distribution generated via multiple backward steps. To achieve this goal, the original training loss of consistency models requires pointwise alignment between the outputs of $f_\theta$ and multiple backward steps, measured for example by an $l_2$-distance. In our formulation, $W_1$-distance is used, which is a discrepancy measure in the distributional sense and ensures the outputs of $f_\theta$ matches that of multiple backward steps in distribution. Note that $W_1$-distance is weaker than the pointwise $l_2$-distance: A small $l_2$-distance implies a small $W_1$-distance. Therefore, our analysis is derived under weaker conditions but covers the stronger $l_2$-distance. In practice, $l_2$-distance is used due to its easy implementation.

%% file: main_theorem1.tex
In this section, we propose our main theorems for the statistical error rates of consistency models, under both settings of distillation training and isolation training. 
\paragraph{Consistency Distillation}
After obtaining the global optima $f_{\hat{\theta}}$ in the optimization problem (\ref{eqn:consistency-objective}), we first construct a baseline consistency model $f_{\theta^*}(\cdot, t)$ induced by natural probability flow ODE solver, named as DDPM solver, whose formulation is presented below. Next, we can upper bound the gap between these two one-step consistency models by applying the optimality condition, with the performance gap represented as
\[W_1\left(f_{\hat{\theta}}(\cdot, T)_\sharp \bX_T, f_{\theta^*}(\cdot, T)_\sharp \bX_T\right).\]
Furthermore, we conclude our main theorem which upper bounds the following statistical error
\begin{equation}
\label{eqn:target}
\cL(\hat{\theta}) = W_1\left(f_{\hat{\theta}}(\cdot, T)_\sharp \mathcal N(0, \bm I), \distri\right).
\end{equation}
by using the bounds of $W_1(\bx_T, \mathcal N(0,\bm I))$ and the approximation error $W_1(f_{\theta^*}(\cdot, T)_\sharp \bX_T, \bX_\varepsilon)$ of the DDPM solver $f_{\theta^*}(\cdot, \cdot)$. Here, the baseline DDPM solver $f_{\theta^*}(\cdot, \cdot)$ is structured as an $N$-layer ResNet \cite{he2016deep} with an inserted pretrained score estimator $s_\phi(\cdot, \cdot)$:
\[f_{\theta^*}(\bx, t) = f_{\theta^*}\left(\hat{\bx}^\phi, ~t - \Delta t\right)~~~~\forall t\in [t_1, T]\]
where $\hat{\bx}^\phi :=\bx + \left(\frac{\beta(t)}{2}\bx + \frac{\beta(t)}{2} s_{\phi}(\bx, t)\right)\Delta t $ is a single-step numerical ODE update from $\bx$ at time $t$. For $\forall t\in [\varepsilon, t_1]$, we let
\[f_{\theta^*}(\bx, t) = \bx + \left(\frac{\beta(t)}{2}\bx + \frac{\beta(t)}{2} s_{\phi}(\bx, t)\right)\cdot (t-\varepsilon). \]
Observe that this structure naturally assures that $f_{\theta^*}(\cdot, \varepsilon) = \mathrm{id}$.

~\\
In this work, we propose an upper bound for the statistical rate of consistency error (\ref{eqn:target}), given a pretrained score estimator $s_\phi$ and a global optimal solution $\hat{\theta}$. Formally, we state our assumptions and main theorem as follows.
\begin{Assumption}[Gaussian tail] 
\label{assump: 1}
For the target distribution $\distri$, it is twice continuously differentiable and it has a Gaussian tail, i.e. there exists positive constants $\alpha_1, \alpha_2 > 0$ such that 
\[\mathbb{P}_{X\sim\distri}\left[\|X\|_2\geqslant R_0\right] \leqslant \mathbb{P}_{Z\sim\mathcal N(0, I)}\left[\|Z\|_2 \geqslant \frac{R_0 - \alpha_1}{\alpha_2}\right]\]
holds for all $R_0 > \alpha_1$. Notice that, this assumption directly leads to the finite second order moment of $\distri$:
\[\mathcal M_2^2 = \mathbb{E}_{X\sim\distri} \|X\|_2^2 < \infty. \]
As we know, the Sub-Gaussian tail is a very mild assumption, encapsulating various practical distributions, such as those with compact support set. Sub-Gaussian tail is also widely studied in existing literature on high-dimensional statistics \cite{wainwright2019high}. 
\end{Assumption}
\begin{Assumption}[Lipschitz score function]
\label{assump: 2}
For any time step $t\in [0, T]$, the score function $\nabla \log p_t (\cdot )$ is $L$-Lipschitz. 
\end{Assumption}
The two assumptions above are mild and have been widely used in relevant works \cite{lyu2023convergence, block2020generative, lee2022convergencea, lee2022convergenceb}. Unlike \cite{block2020generative, de2021diffusion, lee2022convergencea}, we do not need extra conditions on the target distribution such as log-Sobolev inequality or log-concavity, but the Gaussian tail condition is stronger than a bounded second order moment. In this paper, Assumption \ref{assump: 1} is necessary since we need to bound Wasserstein distance with KL divergence. Besides, we can also remove the Lipschitz assumption on the score function by adapting analysis in \cite{benton2023linear}. However, it is only used for technical convenience in bounding the discretization error in Theorems \ref{thm:distillation} and \ref{thm:isolation}, which is only a lower-order term.

\begin{Assumption}[Lipschitz continuity of $f_{\theta^*}$]
\label{assump: 3}
We assume that the baseline consistency model $f_{\theta^*}(\cdot, t)$ is $R$-Lipschitz continuous for $\forall t\in [\varepsilon, T]$.
\end{Assumption}
\begin{Remark}
\label{remark:lipschits-constant}
As \cite{caffarelli1992regularity} proposes, for two distributions $\mu$ and $\rho$ with $\alpha$-Hölder densities and convex support set, there exists a transformation $T^*$ which is $(\alpha+1)$-Hölder smooth, such that $T^*_\sharp \rho = \mu$. This conclusion shows us the existence of transformation with regularity. Assumption \ref{assump: 3} is natural and has been previously used in Assumption 5 of \cite{lyu2023convergence} and Theorem 1 of \cite{song2023consistency}. 
\end{Remark}
\begin{Remark}
\label{remark:more}
Notice that the existence of $f_{\theta^*}(\cdot, t)$ does not imply an access to it. Indeed,  $f_{\theta^*}(\cdot, t)$ is induced by a continuous-time ODE, which is nearly impossible to be queried exactly. Therefore, we need to learn $f_{\theta^*}$ during the training of consistency models. Our Assumption \ref{assump: 3} only asserts that $f_{\theta^*}(\cdot, t)$ is $R$-Lipschitz continuous. However, we do not have access to the ground truth $\theta^*$. 
\end{Remark}
\begin{Assumption}[Bounded coefficient]
\label{assump: 4}
In our variance preserving SDE (\ref{eqn:vp-sde}), the coefficient function $\beta(t)$ is upper and lower bounded by $\overline{\beta}$ and $\underline{\beta}$, such that:
\[\underline{\beta} \leqslant \beta(t) \leqslant \overline{\beta} < \frac{1}{d\log n + d^2 \log(d/\varepsilon)}~~~\text{for}~\forall t\in [\varepsilon, T].\]
\end{Assumption}
Compared with \cite{lyu2023convergence}, we do not require additional assumptions on score estimation error or consistency loss. Actually, bounding these two losses are important parts of our proof.  Now, we introduce our main theorem in below. 
\begin{Theorem}[Main Theorem 1: Distillation]
Under Assumptions \ref{assump: 1} - \ref{assump: 4}, there exists a score estimator $s_\phi(\cdot, t)$ such that the consistency model $f_{\hat{\theta}}(\cdot, t)$ obtained from \eqref{eqn:consistency-objective} satisfies that:
\begin{equation*}
\begin{aligned}
\bE\left[W_1\left(f_{\hat{\theta}}(\cdot, T)_\sharp \mathcal N(0,\bm I), \distri\right)\right]&\lesssim \sqrt{d}R\exp(-\underline{\beta} T/2) + \frac{R\overline{\beta}dLT}{\sqrt{M}} + 6RN'n^{-1/d}\\
&~~~~~~+R\overline{\beta}\sqrt{d} \varepsilon_{\mathrm{score}}\cdot \sqrt{\frac{TN'}{\varepsilon}} + \sqrt{d\overline{\beta}\varepsilon},
\end{aligned}
\end{equation*}
where $R$ is the Lipschitz constraint of the optimization problem (\ref{eqn:consistency-objective}), and the expectation is taken with respect to the choice of dataset $\{\bx^j\}_{j\in [n]}$. $\varepsilon_\mathrm{score}=\mathcal O(n^{-1/(d+5)})$ stands for the score estimation error.
\label{thm:distillation}
\end{Theorem}
We interpret error terms in Theorem \ref{thm:distillation} as follows. $\sqrt{d}R\exp(-\underline{\beta} T/2)$ represents the convergence error of the forward process. $\frac{R\overline{\beta}dLT}{\sqrt{M}}$ is the discretization error of ODE updates.  $6RN'n^{-1/d}$ represents the concentration gap. $R\overline{\beta}\sqrt{d} \varepsilon_{\mathrm{score}}\cdot \sqrt{\frac{TN'}{\varepsilon}}$ is the score estimation error and $\sqrt{d\overline{\beta}\varepsilon}$ is the error caused by early stopping. We show in the following remark that the dominating error term is the score estimation error, with proper choice of hyperparameters.
\begin{Remark}
\label{remark:distillation}
After picking $\overline{\beta}, \underline{\beta}\asymp \frac{1}{d\log n}$, $T = (\log n)^3$, $M=d^2 n^{\frac{1}{d+5}}, N' = \log n$ and $\varepsilon = \sqrt{TN'}n^{-\frac{1}{d+5}}= \log^2 n\cdot n^{-\frac{1}{d+5}}$, we have the bound:
\[\bE\left[W_1\left(f_{\hat{\theta}}(\cdot, T)_\sharp \mathcal N(0,\bm I), \distri\right)\right] \lesssim \sqrt{\log n}\cdot n^{-\frac{1}{2(d+5)}}. \]
Now, we obtain a $\widetilde{\mathcal O}\left(n^{-\frac{1}{2(d+5)}}\right)$ bound for the Wasserstein estimation error of consistency model via distillation, preserving the distribution estimation rate of the vanilla diffusion models as shown in \cite{chen2023score}. This indicates that consistency models maintain the quality of the generated samples, while allowing fast sampling.
\end{Remark}
\begin{Remark}
\label{remark:manifold}
We adopt the nonparametric statistics point of view and the score estimation rate highlights an exponential dependence on the dimension $d$, which is in fact optimal without further assumptions. Nonetheless, it can be reduced in multiple ways: (1) Practical data has rich low-dimensional structures, which is a critical reason why practical diffusion models can be effectively trained. As shown in \cite{chen2023score}, when data has intrinsic subspace structures, the score estimation error only depends on the intrinsic dimension. (2) In parametric settings, we can even obtain a score estimation rate in the order of $\mathrm{poly}(d)/\sqrt{n}$ \cite{yuan2024reward}. We remark that in both cases, the improved convergence rate is tied to data structure assumptions.
\end{Remark}

\paragraph{Consistency Isolation}
Similar to the consistency distillation case, we still need to construct a baseline consistency model $f_{\theta^*}(\cdot, \cdot)$, named as empirical DDPM solver, which replaces the inserted pretrained score model $s_\phi(\bx, t)$ with $\nabla \log \hat{p}_t(\bx)$, the explicit score of a mixture of Gaussian:
\[f_{\theta^*}(\bx, t) = f_{\theta^*}(\hat{\bx}, t-\Delta t)~~~\forall t\in [t_1, T]\]
where $\hat{x} := \bx + \left(\frac{\beta(t)}{2}\bx+\frac{\beta(t)}{2}\nabla\log \hat{p}_t(\bx)\right)\Delta t$ is a single-step numerical ODE update from $\bx$ at time $t$ along the empirical backward ODE (\ref{eqn:backward-total}). For $\forall t\in [\varepsilon, t_1]$, we set
\[f_{\theta^*}(\bx, t) = \bx + \left(\frac{\beta(t)}{2}\bx+\frac{\beta(t)}{2}\nabla\log \hat{p}_t(\bx)\right)\cdot (t-\varepsilon). \]
After obtaining $f_{\hat{\theta}}$ from the optimization problem (\ref{eqn:isolation-objective}), we upper bound the performance gap induced by learned one-step consistency model $f_{\hat{\theta}}(\cdot, T)$ and the empirical DDPM solver $f_{\theta^*}(\cdot, T)$, which is evaluated by $W_1\left(f_{\hat{\theta}}(\cdot, T)_\sharp \cX_T, f_{\theta^*}(\cdot, T)_\sharp \cX_T\right)$. Furthermore, it leads to our main theorem on consistency isolation, which upper bounds the statistical error
\[\mathcal L(\hat{\theta}) = W_1\left(f_{\hat{\theta}}(\cdot, T)_\sharp \mathcal N(0,\bm I), \distri\right).\]
To achieve this result, we require a stronger version of Assumption \ref{assump: 1}, that the target distribution $\distri$ has a bounded support set:
\begin{Assumption}[Bounded support set]
\label{assump: 5}
The target distribution $\distri$ has a bounded support set such that:
\[\mathbb P_{X\sim \distri}\left[\|X\|_2 \leqslant R_0\right] = 1.\]
\end{Assumption}
Here, we require a much stronger assumption than the Gaussian tail because a Lipschitz continuity condition is needed over the empirical score function $\nabla \log \hat{p}_t(\cdot)$ for $\forall t\in [\varepsilon, T]$ to replace \cref{assump: 2}. Now, we state our main theorem on consistency isolation as follows:
\begin{Theorem}[Main Theorem 2: Isolation]
Under Assumptions \ref{assump: 3}- \ref{assump: 5}, the consistency model $f_{\hat{\theta}}(\cdot, t)$ obtained from \cref{eqn:isolation-objective} satisfies that:
\begin{equation*}
\bE\left[W_1\left(f_{\hat{\theta}}(\cdot, T)_\sharp \mathcal N(0,\bm I), \distri\right)\right]\lesssim \sqrt{d} R \exp\left(-\underline{\beta}T/2\right) + Rn^{-1/d} + \frac{d\overline{\beta}R_0^2 T}{\underline{\beta}^2 \varepsilon^2\sqrt{M}}+ \sqrt{d\overline{\beta}\varepsilon},
\end{equation*}
where $R$ is the Lipschitz constraint of the optimization problem (\ref{eqn:isolation-objective}), and the expectation is taken over the dataset. 
\label{thm:isolation}
\end{Theorem}
\begin{Remark}
\label{remark:isolation}
After picking $\overline{\beta}\asymp \frac{1}{d\log n}, \underline{\beta}\asymp \frac{1}{d\log n}$, $\varepsilon = n^{-2/d}, T = d(\log n)^3,$ and $M = d^2 (\log n)^8 \cdot n^{10/d}$, we have the bound:
\[\bE\left[W_1\left(f_{\hat{\theta}}(\cdot, T)_\sharp \mathcal N(0,\bm I), \distri\right)\right] \lesssim  n^{-1/d}.\]
In particular, we obtain a $\widetilde{\mathcal O}\left(n^{-1/d}\right)$ bound for the Wasserstein estimation error of consistency model via isolation. Note that the rate of convergence is not directly comparable to distillation method, due to the distinct training procedure.
\end{Remark}
\begin{Remark}
\label{remark:relaxation}
Assumption \ref{assump: 5} can be straightforwardly relaxed to the sub-Gaussian tail assumption (Assumption \ref{assump: 1}) since the tail shrinks exponentially fast under the sub-Gaussian assumption, which makes it plausible to truncate the data domain with well-controlled truncation errors, and then our analysis reduces to the bounded support case. 
\end{Remark}

%% file: proof-overview.tex
We now present a detailed technical overview for the proof of the statistical error rate for distillation consistency models (Theorem \ref{thm:distillation}). For the proof of isolation consistency models (Theorem \ref{thm:isolation}), it follows very similar ideas and we leave the detailed proof in Appendix \S \ref{sec:sketch-2}. 

~\\
As we state above, our ultimate goal is to upper bound the statistical estimation error $W_1(f_{\hat{\theta}}(\cdot, T)_\sharp \mathcal N(0, \bm I), \distri)$, the distance between the true distribution $\distri$ and standard Gaussian pushed forward by our learned one-step consistency model $f_{\hat{\theta}}(\cdot, T)$ by distillation. To achieve this, we construct a DDPM solver $f_{\theta^*}(\cdot, \cdot)$ which is assumed to be $R$-Lipschitz continuous at all time steps $t\in [\varepsilon, T]$, and upper bound the performance gap between $f_{\hat{\theta}}$ and $f_{\theta^*}$. 

~\\
In the first step, we study the approximation properties of score estimation, and our purpose is to show the existence of a score network $s_\phi(\bx, t)$ with small approximation error $\mathbb E\|s_\phi(\bx, t)-\nabla \log p_t(\bx)\|^2$. With the score approximation error bounded, we can conclude the proximity between the true backward probability ODE and that with pretrained score model inserted. 

~\\
Next, we aim to bound the performance gap between the learned one-step consistency model $f_{\hat{\theta}}$ and the DDPM solver $f_{\theta^*}$. According to the training objective (\ref{eqn:consistency-objective}) as well as Assumption \ref{assump: 3} which makes $f_{\theta^*}$ also included in the constraint set $\mathrm{Lip}(R)$, we can apply the optimality inequality 
\begin{equation}
\label{eqn: optimality}
\begin{aligned}
\cL_{\mathrm{CD}}^N(\hat{\theta};\phi)\leqslant \cL_{\mathrm{CD}}^N(\theta^*;\phi).
\end{aligned} 
\end{equation}
Through some mathematical calculation, we show that the performance gap is directly relevant to the concentration gap between empirical and population distributions, as well as the error caused by the numerical ODE update. There are two main types of error taking place during the ODE update, which are the discretization error and the score estimation error. The former one is directly relevant to the length of time sub-intervals $\Delta t$ while the bound of the latter one is already solved in our first step. 

~\\
After that, we finally come to our main theorem upper bounding the statistical error \eqref{eqn:target}. It can be smoothly obtained by combining the performance gap, the tail bounds $W_1(\bX_T, \mathcal N(0,\bm I))$, $W_1(\bX_\varepsilon, \distri)$ as well as the estimation error of the DDPM solver $W_1(f_{\theta^*}(\cdot, T)_\sharp \bX_T, \bX_\varepsilon)$. 

~\\
In contrast, for the proof of Theorem \ref{thm:isolation}, the major difference in the isolation setting is that there is no score estimation error since the isolation training does not involve any pretrained score models. For each ODE update, discretization is the only error that takes place. Another technical difficulty is to guarantee the Lipschitz continuity of the empirical score functions involved in the backward process. 

%% file: approx-1.tex
In \cite{chen2023score}, the authors introduce the $l$-layer ReLU neural network class $\mathrm{NN}(l, M, J, K, \kappa, \gamma, \gamma_t)$ as follows and propose a score approximation error, with the result shown in Lemma \ref{lemma:score0}. After removing the properties we do not need, we can make the following conclusion:
\begin{Lemma}
\label{lemma:score}
There exists a score estimator function $s_\phi(\cdot, \cdot)$ in the class of neural networks, such that: (1) $s_\phi(\cdot, t)$ is $L_{\score}$-Lipschitz continuous for any given $t\in[\varepsilon, T]$ where $L_{\score}=\mathcal O(10d(1+L))$; (2) $\|s_\phi(\bx, t)\|_2 \leqslant U_{\score}$ holds for $\forall \bx\in\bR^d, t\in [\varepsilon, T]$ where $U_{\score} = \mathcal O(2d\log n + 2d^2\log(d/\varepsilon))$; (3) The mean integrated squared error can be upper bounded by:
\begin{equation*}
\begin{aligned}
\frac{1}{t_b-t_a} \int_{t_a}^{t_b} \left\|s_\phi(\cdot, t) - \nabla\log p_t(\cdot)\right\|_{L^2(\bX_t)}^2 \rd t= \widetilde{O}\left(\frac{1}{\varepsilon} n^{-\frac{2}{d+5}}\right)~~~\forall \varepsilon\leqslant t_a < t_b \leqslant T.
\end{aligned}
\end{equation*}
\end{Lemma}
According to this result, we can use the method of induction to provide a loose upper bound for the Lipschitz constant of $f_{\theta^*}(\cdot, \cdot)$, the baseline DDPM solver with score estimator $s_\phi(\cdot, \cdot)$ injected. 
\begin{Corollary}
\label{lemma:lip-star}
Assume the information decay rate $\beta(t)$ in \cref{eqn:vp-sde} is bounded as $\underline{\beta}\leqslant \beta(t)\leqslant \overline{\beta}$ 
for $\forall t\in[\varepsilon, T]$, then the trivial upper bound for the Lipschitz constant of $f_{\theta^*}(\cdot, t)$ is $\exp(Cd\overline{\beta} T)$ for any given $t$. Here $C=10(1+L)$ is a pure constant. 
\end{Corollary}
\begin{proof}
Detailed proof is left in Appendix \S \ref{sec:C.2}.
\end{proof}
By Lemma \ref{lemma:score}, we get the approximation error bound of score model, which is part of the performance gap between $f_{\hat{\theta}}$ and $f_{\theta^*}$. In the next part, we apply the optimality inequality \eqref{eqn: optimality} and decompose the consistency loss into several error terms which are easier to analyze. We will also show that these terms stand for the concentration gap and the numerical ODE update error.

%% file: decomposition-consistency.tex
According to the structure of $f_{\theta^*}$, we have
\[f_{\theta^*}(\cdot, \tau_k) = f_{\theta^*}(\cdot, \tau_{k-1})\circ G_{(M)}(\cdot, \tau_k; \phi). \]
Denote $\hat{\bX}_{\tau_{k-1}}^{\phi, M} := G_{(M)}(\cdot, \tau_k; \phi)$ as the underlying distribution of $\hat{\bx}_{\tau_{k-1}}^{\phi, M} = G_{(M)}(\bx_{\tau_k}, \tau_k; \phi)$ where $\bx_{\tau_k}\sim \bX_{\tau_k}$, then it holds by definition that:
\begin{equation}
\label{eqn:thetastar}
f_{\theta^*}(\cdot, \tau_k)_\sharp \bX_{\tau_k} \overset{\rm law}{=} f_{\theta^*}(\cdot, \tau_{k-1})_\sharp \hat{\bX}_{\tau_{k-1}}^{\phi,M}~~~\forall k\in[N'].
\end{equation}
This equation lays the foundation of recursive analysis between adjacent time steps. After Combining with the optimality inequality \eqref{eqn: optimality}, we can decompose the performance gap between $f_{\hat{\theta}}$ and $f_{\theta^*}$ (also known as consistency loss) into four loss terms, which is shown in the following lemma.  
\begin{Lemma}
\label{lemma: decompose}
We can upper bound the consistency loss as:
\begin{equation}
\label{eqn: decompose}
W_1\left(f_{\hat{\theta}}(\cdot, T)_\sharp \bX_T, f_{\theta^*}(\cdot, T)_\sharp \bX_T\right) \leqslant  I_1 + I_2 + I_3 + I_4 .
\end{equation}
Here, the four loss terms $I_i~(1\leqslant i\leqslant 4)$ have their formulations as follows:
\begin{align*}
I_1 &:= \sum_{k=1}^{N'} W_1\left(f_{\hat{\theta}}(\cdot, \tau_{k-1})_\sharp \bX_{\tau_{k-1}}, f_{\hat{\theta}}(\cdot, \tau_{k-1})_\sharp \widehat{\bX}_{\tau_{k-1}}^{\phi, M}\right),\\
I_2 &:= \sum_{k=1}^{N'} W_1\left(f_{\theta^*}(\cdot, \tau_{k-1})_\sharp \bX_{\tau_{k-1}}, f_{\theta^*}(\cdot, \tau_{k-1})_\sharp \widehat{\bX}_{\tau_{k-1}}^{\phi, M}\right),\\
I_3 &:= \sum_{k=1}^{N'} \Big[W_1\left(f_{\hat{\theta}}(\cdot, \tau_k)_\sharp \bX_{\tau_k}, f_{\hat{\theta}}(\cdot, \tau_{k-1})_\sharp \hat{\bX}_{\tau_{k-1}}^{\phi,M}\right)- W_1\left(f_{\hat{\theta}}(\cdot, \tau_k)_\sharp \cX_{\tau_k}, f_{\hat{\theta}}(\cdot, \tau_{k-1})_\sharp \hat{\cX}_{\tau_{k-1}}^{\phi,M}\right)\Big]\\
I_4 &:= \sum_{k=1}^{N'} \Big[W_1\left(f_{\theta^*}(\cdot, \tau_k)_\sharp \cX_{\tau_k}, f_{\theta^*}(\cdot, \tau_{k-1})_\sharp \hat{\cX}_{\tau_{k-1}}^{\phi,M}\right)- W_1\left(f_{\theta^*}(\cdot, \tau_k)_\sharp \bX_{\tau_k}, f_{\theta^*}(\cdot, \tau_{k-1})_\sharp \hat{\bX}_{\tau_{k-1}}^{\phi,M}\right)\Big].
\end{align*}
\end{Lemma}
\begin{proof}
Detailed proof is left in Appendix \S \ref{sec:C.3}.  
\end{proof}
As we can see, both $I_1, I_2$ show the multi-step discretization error of ODE solver and both $I_3, I_4$ show the concentration gap between empirical and population Wasserstein-1 distances. Next, we start with $I_1, I_2$. The technical difficulties on upper bound these terms come from two aspects. One is to bound the KL divergence between the true ODE flow measure and that with pretrained score function $s_\phi(\cdot, \cdot)$ inserted. The other is to bound Wasserstein distance with KL divergence, which is impossible in general but achievable under Gaussian tail condition (Assumption \ref{assump: 1}).
After overcoming these obstacles, we prove the following lemma. 
\begin{Lemma}
\label{lemma:12}
Under the Assumption \ref{assump: 1}-\ref{assump: 4}, we can upper bound $I_1, I_2$ introduced in Lemma \ref{lemma: decompose} as:
\[I_1 + I_2 \lesssim R\overline{\beta}dL\cdot \frac{T}{\sqrt{M}} + R\overline{\beta}\sqrt{d} n^{-\frac{1}{d+5}}\cdot \sqrt{\frac{TN'}{\varepsilon}}. \]
Here, $R$ is the Lipschitz constraint in \eqref{eqn:consistency-objective}. 
\end{Lemma}
\begin{proof}
Detailed proof is left in Appendix \S \ref{sec:C.4}. 
\end{proof}
Next, we upper bound $I_3, I_4$. After transforming them into the Wasserstein distance between empirical and population distributions, we prove the following lemma.
\begin{Lemma}
\label{lemma:34}
Under the Assumption \ref{assump: 1}-\ref{assump: 4}, we can upper bound $I_3, I_4$ introduced in Lemma \ref{lemma: decompose} as:
\[\bE \left[I_3 + I_4 \right] \leqslant 6RN'\cdot n^{-1/d}.\]
Here, the expectation is taken with respect to the randomness of dataset $\{\bx^1, \bx^2, \ldots, \bx^n\}$. 
\end{Lemma}
\begin{proof}
Detailed proof is left in Appendix \S \ref{sec:C.5}.
\end{proof}
Now we can combine all the results above and get:
\begin{equation}
\label{eqn:combine-I1234}
\begin{aligned}
&\bE \left[W_1\left(f_{\hat{\theta}}(\cdot, T)_\sharp \bX_T, f_{\theta^*}(\cdot, T)_\sharp \bX_T\right)\right] \\
&~\lesssim \frac{R\overline{\beta}dLT}{\sqrt{M}} + R\overline{\beta} n^{-\frac{1}{d+5}} \sqrt{\frac{dTN'}{\varepsilon}} + 6RN'n^{-1/d}
\end{aligned}
\end{equation}
holds under Assumption \ref{assump: 1}-\ref{assump: 4}. 

%% file: proof-thm1.tex
In order to bound the statistical error $\cL(\hat{\theta})$ defined in \eqref{eqn:target}, we still need to bound two additional loss terms: $W_1\left(f_{\hat{\theta}}(\cdot, T)_\sharp \mathcal N(0, \bm I), f_{\theta^*}(\cdot, T)_\sharp \bX_T\right)$ and the estimation error of DDPM solver $W_1\left(f_{\theta^*}(\cdot, T)_\sharp \bX_T, \distri\right)$. Since $f_{\hat{\theta}}(\cdot, T)$ is $R$-Lipschitz continuous, we have
\[W_1\left(f_{\hat{\theta}}(\cdot, T)_\sharp \bX_T, f_{\hat{\theta}}(\cdot, T)_\sharp \mathcal N(0, \bm I)\right) \leqslant R\cdot W_1(\bX_T, \mathcal N(0,\bm I)).\]
Therefore, we first need to bound $W_1(\bX_T, \mathcal N(0,\bm I))$ in the following lemma.
\begin{Lemma}
\label{lemma:gaussian}
For the distribution $\bX_T$, its Wasserstein distance from the standard Gaussian $\mathcal N(0, \bm I)$ can be upper bounded as:
\[W_1(\bX_T, \mathcal N(0,\bm I)) \lesssim \sqrt{d}\exp(-\underline{\beta} T/2). \]
\end{Lemma}
\begin{proof}
Detailed proof is left in Appendix \S \ref{sec:C.6}. 
\end{proof}
Next, we bound $W_1\left(f_{\theta^*}(\cdot, T)_\sharp \bX_T, \distri\right)$, which requires an extension on the existing result on DDPM estimation error (Theorem 2 of \cite{chen2022sampling}) as well as the technique of bounding Wasserstein distance with KL divergence.
\begin{Lemma}
\label{lemma:ddpm-final}
Under Assumption \ref{assump: 1}-\ref{assump: 4}, we bound the estimation error of DDPM solver as:
\begin{equation*}
W_1\left(f_{\theta^*}(\cdot, T)_\sharp \bX_T, \distri\right) \lesssim \overline{\beta}Ld \sqrt{T\Delta t} + \overline{\beta}\sqrt{\frac{dT}{\varepsilon}} n^{-\frac{1}{d+5}} + \sqrt{d\overline{\beta}\varepsilon}. 
\end{equation*}
\end{Lemma}
\begin{proof}
Detailed proof is left in Appendix \S \ref{sec:C.7}.  
\end{proof}
Now, after summing up Lemma \ref{lemma:gaussian}, \ref{lemma:ddpm-final} and Equation (\ref{eqn:combine-I1234}) together, we finally come to our main theorem \ref{thm:distillation}:
\begin{equation*}
\begin{aligned}
\bE\left[W_1\left(f_{\hat{\theta}}(\cdot, T)_\sharp \mathcal N(0,\bm I), \distri\right)\right] &\lesssim \sqrt{d}R\exp(-\underline{\beta} T/2) + \frac{R\overline{\beta}dLT}{\sqrt{M}}+R\overline{\beta}n^{-\frac{1}{d+5}} \sqrt{\frac{dTN'}{\varepsilon}}\\
&~~~~+ \sqrt{d\overline{\beta}\varepsilon}+6RN'n^{-1/d}.
\end{aligned}
\end{equation*}

%% file: conclusion.tex
In this paper, we have provided the first statistical theory of consistency diffusion models. In particular, we have formulated the consistency models' training as a Wasserstein discrepancy minimization problem. Further, we have established sample complexity bounds for consistency models in estimating nonparametric data distributions. The obtained convergence rate closely matches the vanilla diffusion models, indicating consistency models boost the sampling speed without significantly scarifying the sample generation quality. Our analyses have covered both the distillation and isolation methods for training consistency models.

%% file: appendix.tex
\section{Proofs in Section \ref{sec:consistency_train}}
\subsection{Proof of Lemma \ref{lemma:isolation-pre}}
\label{sec:A.1}
When $\bx_0$ follows the empirical distribution $\hat{\distri} = \frac1n\sum_{j=1}^n \delta_{\bx^j}$, then the posterior distribution $p(\bx_0\mid \bx_t)$ for a given $\bx_t\sim \mathcal N(m(t)\bx_0, \sigma(t)^2\bm I)$ can be simply represented as:
\[p(\bx_0 = \bx^j \mid \bx_t) \propto \exp\left(-\frac{\|m(t)\bx^j - \bx_t\|^2}{2\sigma(t)^2}\right), \]
which leads to the following posterior mean:
\[\bE_{\bx_0\sim \hat{\distri}} [\bx_0\mid \bx_t] = \sum_{j=1}^n \bx^j\cdot p(\bx_0=\bx^j\mid \bx_t) =  \frac{\sum_{j=1}^n \bx^j \cdot\exp\left(-\frac{\|m(t)\bx^j - \bx_t\|^2}{2\sigma(t)^2}\right)}{\sum_{j=1}^n \exp\left(-\frac{\|m(t)\bx^j - \bx_t\|^2}{2\sigma(t)^2}\right)}. \]
Therefore, the score function has the following unbiased estimation:
\begin{equation}
\label{eqn:isolation-equivalence}
\begin{aligned}
\nabla\log p_t(\bx_t) &\approx -\bE_{\bx_0\sim \hat{\distri}}\left[\frac{\bx_t - m(t)\bx_0}{\sigma(t)^2}\bigg| \bx_t\right] = \frac{\sum_{j=1}^n -\frac{\bx_t - m(t)\bx^j}{\sigma(t)^2} \cdot\exp\left(-\frac{\|m(t)\bx^j - \bx_t\|^2}{2\sigma(t)^2}\right)}{\sum_{j=1}^n \exp\left(-\frac{\|m(t)\bx^j - \bx_t\|^2}{2\sigma(t)^2}\right)}\\
&= \frac{\nabla_{\bx_t} \sum_{j=1}^n \exp\left(-\frac{\|m(t)\bx^j - \bx_t\|^2}{2\sigma(t)^2}\right)}{\sum_{j=1}^n \exp\left(-\frac{\|m(t)\bx^j - \bx_t\|^2}{2\sigma(t)^2}\right)} = \nabla_{\bx_t} \log\left[\frac1n\sum_{j=1}^n \exp\left(-\frac{\|m(t)\bx^j - \bx_t\|^2}{2\sigma(t)^2}\right)\right]. 
\end{aligned}
\end{equation}
Notice that, $\frac1n\sum_{j=1}^n \exp\left(-\frac{\|m(t)\bx^j - \bx_t\|^2}{2\sigma(t)^2}\right)$ is exactly the density of $\mathcal X_t = m(t)\hat{\distri} \star \mathcal N(0, \sigma(t)^2)$.

\section{Some Useful Lemmas}
\input{helper_proof}

\section{Proofs in Section \ref{sec: sketch}}
\subsection{Approximation Error for Score Approximation}
\label{sec:C.1}
\begin{Lemma}
\label{lemma:score0}
Define the $l$-layer ReLU network class $\mathrm{NN}(l, M, J,K,\kappa, \gamma, \gamma_t)$ as follows:
\begin{equation*}
\begin{aligned}
&~~\mathrm{NN}(l, M, J, K, \kappa, \gamma, \gamma_t) = \\
& \Big\{s(\bz,t)=W_l\sigma(\ldots \sigma(W_1 [\bz^\top, t]^\top)\ldots) + b_l~\mid\\
&~~~~~~\text{Network width is bounded by }M; ~\sup_{\bz,t}\|f(\bz,t)\|_2\leqslant K;\\
&~~~~~~\max_i \max(\|b_i\|_\infty, \|W_i\|_\infty)\leqslant \kappa;~\sum_{i=1}^l (\|W_i\|_0 + \|b_i\|_0)\leqslant J;\\
&~~~~~~~\|s(\bz, t) - s(\bz', t)\|_2 \leqslant \gamma\|\bz-\bz'\|_2~\text{holds for }\forall \bz,\bz',t;\\
&~~~~~~~\|s(\bz, t) - s(\bz, t')\|_2 \leqslant \gamma_t|t-t'|~\text{holds for }\forall \bz,t,t'\Big\}.
\end{aligned}
\end{equation*}
As we see, all the neural networks in this class has bounded function value, bounded weights, bounded width and Lipschitz continuity. Given an approximation error $\delta > 0$, we choose the network hyperparameter as:
\[l = \mathcal O(d + \log(1/\delta)),~ K=\mathcal O(2d^2\log(d/\varepsilon\delta)), ~\gamma = 10d(1+L), ~\gamma_t=10\tau,\]
\[M = \mathcal O\left((1+L)^d T\tau d^{d/2+1}\delta^{-(d+1)}\log^{d/2}(d/\varepsilon\delta)\right),\]
\[J = \mathcal O\left((1+L)^d T\tau d^{d/2+1}\delta^{-(d+1)}\log^{d/2}(d/\varepsilon\delta)(d+\log(1/\delta))\right),\]
\[\kappa = \mathcal O\left(\max\left(2(1+L)\sqrt{d\log(d/\varepsilon\delta)}, T\tau\right)\right)\]
where $\delta$ is chosen as $\delta = n^{-\frac{1-\tau(n)}{d+5}}$ for $\tau(n)=\frac{d\log\log n}{\log n}$ and 
$$\tau := \sup_t \sup_{\|\bz\|_\infty \leqslant \sqrt{d\log(d/\varepsilon\delta)}} \left\|\frac{\partial}{\partial t} \left[ \sigma(t)^2 \nabla \log p_t(\bz)\right]\right\|_2.$$
After choosing There exists $s_\phi\in \mathrm{NN}$ such that with probability at least $1-\frac1n$, it holds that
\[\frac{1}{t_b-t_a} \int_{t_a}^{t_b} \left\|s_\phi(\cdot, t) - \nabla\log p_t(\cdot)\right\|_{L^2(\bX_t)}^2 \rd t = \widetilde{O}\left(\frac{1}{\varepsilon} n^{-\frac{2}{d+5}}\right)~~~\forall \varepsilon\leqslant t_a < t_b \leqslant T.\]
\end{Lemma}
\subsection{Proof of Corollary \ref{lemma:lip-star}}
\label{sec:C.2}
For any given $0\leqslant k < N$, denote $L_k$ to be the Lipschitz constant of $f_\theta^*(\cdot, t)$ when $t\in[t_k, t_{k+1}]$. If we treat $\bx^\phi$ as a function over $\bx$, its Lipschitz constant is no larger than 
\[1 + \overline{\beta}(1+L_\score)\Delta t/ 2 \leqslant 1 + Cd\overline{\beta}\Delta t\]
where $C = 10(1+L)$ is a pure constant. This is also the upper bound of $L_1$. Here, we use the result in \cref{lemma:score} that $L_\score = \mathcal O(10d(1+L))$. Therefore:
\[L_{k+1} \leqslant (1+Cd\overline{\beta}\Delta t)L_k\]
holds according to the recursive formulation of $f_{\theta^*}(\cdot,\cdot)$, which leads to the conclusion that, the Lipschitz constant of $f_{\theta^*}(\cdot,t)$ is no larger than:
\[(1 + Cd\overline{\beta}\Delta t)^N = (1 + Cd\overline{\beta}\Delta t)^{T/\Delta t} \leqslant \exp(Cd\overline{\beta}T),\]
which proves the conclusion. 

\subsection{Proof of Lemma \ref{lemma: decompose}}
\label{sec:C.3}
As described in the lemma, we recall that
\begin{equation*}
\begin{aligned}
I_1 &:= \sum_{k=1}^{N'} W_1\left(f_{\hat{\theta}}(\cdot, \tau_{k-1})_\sharp \bX_{\tau_{k-1}}, f_{\hat{\theta}}(\cdot, \tau_{k-1})_\sharp \widehat{\bX}_{\tau_{k-1}}^{\phi, M}\right),\\
I_2 &:= \sum_{k=1}^{N'} W_1\left(f_{\theta^*}(\cdot, \tau_{k-1})_\sharp \bX_{\tau_{k-1}}, f_{\theta^*}(\cdot, \tau_{k-1})_\sharp \widehat{\bX}_{\tau_{k-1}}^{\phi, M}\right),\\
I_3 &:= \sum_{k=1}^{N'} \left[W_1\left(f_{\hat{\theta}}(\cdot, \tau_k)_\sharp \bX_{\tau_k}, f_{\hat{\theta}}(\cdot, \tau_{k-1})_\sharp \hat{\bX}_{\tau_{k-1}}^{\phi,M}\right) - W_1\left(f_{\hat{\theta}}(\cdot, \tau_k)_\sharp \cX_{\tau_k}, f_{\hat{\theta}}(\cdot, \tau_{k-1})_\sharp \hat{\cX}_{\tau_{k-1}}^{\phi,M}\right)\right]\\
I_4 &:= \sum_{k=1}^{N'} \left[W_1\left(f_{\theta^*}(\cdot, \tau_k)_\sharp \cX_{\tau_k}, f_{\theta^*}(\cdot, \tau_{k-1})_\sharp \hat{\cX}_{\tau_{k-1}}^{\phi,M}\right) - W_1\left(f_{\theta^*}(\cdot, \tau_k)_\sharp \bX_{\tau_k}, f_{\theta^*}(\cdot, \tau_{k-1})_\sharp \hat{\bX}_{\tau_{k-1}}^{\phi,M}\right)\right].
\end{aligned}
\end{equation*}
First, from the optimality condition (\ref{eqn: optimality}) and the structure of $f_{\theta}^*$ (\ref{eqn:thetastar}), we have:
\begin{equation*}
\begin{aligned}
&~~~~\sum_{k=1}^{N'} W_1 \left(f_{\hat{\theta}}(\cdot, \tau_k)_\sharp \cX_{\tau_k}, f_{\hat{\theta}}(\cdot, \tau_{k-1})_\sharp \hat{\cX}_{\tau_{k-1}}^{\phi, M}\right) \leqslant \sum_{k=1}^{N'} W_1 \left(f_{\theta^*}(\cdot, \tau_k)_\sharp \cX_{\tau_k}, f_{\theta^*}(\cdot, \tau_{k-1})_\sharp \hat{\cX}_{\tau_{k-1}}^{\phi, M}\right)\\
&\leqslant \sum_{k=1}^{N'} \left[W_1\left(f_{\theta^*}(\cdot, \tau_k)_\sharp \cX_{\tau_k}, f_{\theta^*}(\cdot, \tau_{k-1})_\sharp \hat{\cX}_{\tau_{k-1}}^{\phi,M}\right) - W_1\left(f_{\theta^*}(\cdot, \tau_k)_\sharp \bX_{\tau_k}, f_{\theta^*}(\cdot, \tau_{k-1})_\sharp \hat{\bX}_{\tau_{k-1}}^{\phi,M}\right)\right]\\
&~~~+ \sum_{k=1}^{N'} W_1 \left(f_{\theta^*}(\cdot , \tau_k)_\sharp \bX_{\tau_k}, f_{\theta^*}(\cdot, \tau_{k-1})_\sharp \hat{\bX}_{\tau_{k-1}}^{\phi, M}\right) = I_4.
\end{aligned}
\end{equation*}
Then, we can immediately conclude that:
\begin{equation}
\label{eqn:B-1}
\sum_{k=1}^{N'} W_1 \left(f_{\hat{\theta}}(\cdot, \tau_k)_\sharp \bX_{\tau_k}, f_{\hat{\theta}}(\cdot, \tau_{k-1})_\sharp \hat{\bX}_{\tau_{k-1}}^{\phi, M}\right)\leqslant I_3 + I_4. 
\end{equation}
Again by using \cref{eqn:thetastar}, we know that for $\forall k\in [N']$:
\begin{equation}
\label{eqn:B-2}
\begin{aligned}
&~~~W_1\left(f_{\hat{\theta}}(\cdot, \tau_k)_\sharp \bX_{\tau_k}, f_{\theta^*}(\cdot, \tau_k)_\sharp \bX_{\tau_k}\right)=W_1\left(f_{\hat{\theta}}(\cdot, \tau_k)_\sharp \bX_{\tau_k}, f_{\theta^*}(\cdot, \tau_{k-1})_\sharp \hat{\bX}_{\tau_{k-1}}^{\phi,M}\right)\\
&\leqslant W_1\left(f_{\hat{\theta}}(\cdot, \tau_k)_\sharp \bX_{\tau_k}, f_{\hat{\theta}}(\cdot, \tau_{k-1})_\sharp \hat{\bX}_{\tau_k}^{\phi,M}\right)+W_1\left(f_{\hat{\theta}}(\cdot, \tau_{k-1})_\sharp \bX_{\tau_{k-1}}, f_{\hat{\theta}}(\cdot, \tau_{k-1})_\sharp \hat{\bX}_{\tau_{k-1}}^{\phi,M}\right)\\
&~~+ W_1\left(f_{\hat{\theta}}(\cdot, \tau_{k-1})_\sharp \bX_{\tau_{k-1}}, f_{\theta^*}(\cdot, \tau_{k-1})_\sharp \bX_{\tau_{k-1}}\right) + W_1\left(f_{\theta^*}(\cdot, \tau_{k-1})_\sharp \bX_{\tau_{k-1}}, f_{\theta^*}(\cdot, \tau_{k-1})_\sharp \hat{\bX}_{\tau_{k-1}}^{\phi,M}\right) 
\end{aligned}
\end{equation}

Then, after summing over $k=1,2,\ldots, N'$ and telescoping, we have:
\begin{equation*}
\begin{aligned}
&~~~W_1\left(f_{\hat{\theta}}(\cdot, T)_\sharp \bX_T, f_{\theta^*}(\cdot, T)_\sharp \bX_T\right) \leqslant \sum_{k=1}^{N'} W_1\left(f_{\hat{\theta}}(\cdot, \tau_k)_\sharp \bX_{\tau_k}, f_{\hat{\theta}}(\cdot, \tau_{k-1})_\sharp \hat{\bX}_{\tau_{k-1}}^{\phi,M}\right)\\
&+ \sum_{k=1}^{N'} W_1\left(f_{\hat{\theta}}(\bX_{\tau_{k-1}}, \tau_{k-1}), f_{\hat{\theta}}(\cdot, \tau_{k-1})_\sharp \hat{\bX}_{\tau_{k-1}}^{\phi,M}\right) + \sum_{k=1}^{N'}W_1\left(f_{\theta^*}(\cdot, \tau_{k-1})_\sharp \bX_{\tau_{k-1}}, f_{\theta^*}(\cdot, \tau_{k-1})_\sharp \hat{\bX}_{\tau_{k-1}}^{\phi,M}\right)\\
&= \sum_{k=1}^{N'} W_1\left(f_{\hat{\theta}}(\cdot, \tau_k)_\sharp \bX_{\tau_k}, f_{\hat{\theta}}(\cdot, \tau_{k-1})_\sharp \hat{\bX}_{\tau_{k-1}}^{\phi,M}\right) + I_1 + I_2 \leqslant I_1 + I_2 + I_3 + I_4.
\end{aligned}
\end{equation*}
Here, we apply \cref{eqn:B-1} to the last line, and finally we come to our conclusion. 

\subsection{Proof of Lemma \ref{lemma:12}}
\label{sec:C.4}
\input{I1-I2-bound}

\subsection{Proof of Lemma \ref{lemma:34}}
\label{sec:C.5}
\input{I3-I4-bound}

\subsection{Proof of Lemma \ref{lemma:gaussian}}
\label{sec:C.6}
From the proof of Lemma \ref{lemma: others-1}, we know that 
\[W_1\left(P\star \mathcal N(0,\sigma^2 \bm I), Q\star \mathcal N(0,\sigma^2 \bm I)\right) \leqslant W_1(P,Q) \]
holds for any distribution pair $(P,Q)$ and $\sigma > 0$. For distribution $\bX_T$ and $\mathcal N(0,\bm I)$, we have:
\[\bX_T = \left(m(T)\cdot \distri\right) \star \mathcal N(0, \sigma(T)^2)~~\text{and}~~\mathcal N(0,\bm I) = \left(m(T)\cdot \mathcal N(0,\bm I)\right) \star \mathcal N(0, \sigma(T)^2)\]
since $m(T)^2 + \sigma(T)^2 = 1$. Therefore:
\[W_1(\bX_T, \mathcal N(0, \bm I)) \leqslant m(T) \cdot W_1(\distri, \mathcal N(0, \bm I)). \]
According to \cref{assump: 2}, we know that $W_1(\distri, \mathcal N(0, \bm I))$ is finite and furthermore $W_1(\distri, \mathcal N(0, \bm I)) \lesssim \sqrt{d}$. Besides,
\[m(T) = \exp\left(-\frac12\int_0^T \beta(s) \rd s\right) \leqslant \exp(-\underline{\beta} T/2).\]
To sum up, we conclude that:
\[W_1(\bX_T, \mathcal N(0, \bm I)) \lesssim \sqrt{d}\exp(-\underline{\beta} T/2).\]

\subsection{Proof of Lemma \ref{lemma:ddpm-final}}
\label{sec:C.7}
\begin{Lemma}[DDPM]
\label{lemma:ddpm}
Under Assumption \ref{assump: 1} and \ref{assump: 2}, when the step size $\Delta t < 1/L$, it holds that:
\[\mathrm{KL}\left(f_{\theta^*}(\cdot, T)_\sharp \bX_T, \bX_\varepsilon\right) \lesssim \overline{\beta}^2 L^2 T (d\Delta t + \mathcal M_2^2 \Delta t^2) + \overline{\beta}^2\int_\varepsilon^T 
\left\|s_\phi(\cdot, t) - \nabla\log p_t(\cdot)\right\|_{L^2(\bX_t)}^2 \rd t\]
\end{Lemma}
Since both $\bX_\varepsilon$ and $f_{\theta^*}(\cdot, T)_\sharp \bX_T$ have Gaussian tail, we apply \cref{lemma: W1-TV-1} and the score integrated error (\cref{lemma:score}), then we conclude that:
\[W_1\left(f_{\theta^*}(\cdot, T)_\sharp \bX_T, \bX_\varepsilon\right) \lesssim \sqrt{d}\cdot \sqrt{\mathrm{KL}\left(f_{\theta^*}(\cdot, T)_\sharp \bX_T, \bX_\varepsilon\right)}\lesssim \overline{\beta}Ld \sqrt{T\Delta t} + \overline{\beta}\sqrt{d}\cdot \sqrt{\frac{T}{\varepsilon}} n^{-\frac{1}{d+5}}.\]

Finally, we just need to bound $W_1(\bX_\varepsilon, \distri)$, which is stated in the following lemma.
\begin{Lemma}
For the distributions $\distri$ and $\bX_\varepsilon = \left(m(\varepsilon)\cdot \distri\right) \star \mathcal N(0, \sigma(\varepsilon)^2 \bm I)$, its Wasserstein-1 distance with $\distri$ can be upper bounded as:
\[W_1(\bX_\varepsilon, \distri) \lesssim \sqrt{d\overline{\beta}\varepsilon}. \]
\label{lemma:distri-varepsilon}
\end{Lemma}
\begin{proof}
Notice that $\bX_\varepsilon = (m(\varepsilon)\cdot \distri)\star \mathcal N(0, \sigma(\varepsilon)^2)$ where
\[m(\varepsilon) = \exp\left(-\frac12 \int_0^\varepsilon \beta(s) \rd s\right)\geqslant \exp(-\overline{\beta}\varepsilon /2) \geqslant 1-\overline{\beta}\varepsilon /2,\]
and $\sigma(\varepsilon)^2 = 1-m(\varepsilon)^2 \leqslant 2(1-m(\varepsilon)) \leqslant \overline{\beta}\varepsilon$. Then, it holds that:
\begin{equation*}
\begin{aligned}
W_1(\bX_\varepsilon, \distri) &\leqslant W_1\left(\distri, m(\varepsilon)\cdot \distri\right) + W_1\left(m(\varepsilon)\cdot \distri, \bX_\varepsilon\right)\\
& \leqslant \sup_{\mathrm{Lip}(f)\leqslant 1} \bE_{\bx\sim \distri} \left[f(\bx)-f(m(\varepsilon)\cdot\bx)\right] + W_1(\delta_{\left\{0\right\}}, \mathcal N(0,\sigma(\varepsilon)^2))\\
& \leqslant (1-m(\varepsilon))\cdot \bE_{\bx\sim\distri} \|\bx\|_2 + \sigma(\varepsilon)\cdot \bE_{\bz\sim\mathcal N(0,\bm I)} \|z\|_2 \\
& \leqslant (1-m(\varepsilon))\cdot \mathcal M_2 + \sigma(\varepsilon)\cdot \bE_{\bz\sim\mathcal N(0,\bm I)} \|\bz\|_2 \\
& \leqslant \overline{\beta}\varepsilon / 2 \cdot \mathcal M_2 + \sqrt{\overline{\beta}\varepsilon}\cdot \sqrt{d} \lesssim \sqrt{d\overline{\beta}\varepsilon},
\end{aligned}
\end{equation*}
which comes to our conclusion.   
\end{proof}

\section{Proof Sketch for Consistency Isolation}
\label{sec:sketch-2}
\input{decomposition-isolation}

%% file: helper_proof.tex
\label{sec:B}
In this section, we introduce some lemmas directly related to the Girsanov's theorem and techniques from \cite{chen2022sampling}. We also propose some propositions on Gaussian tails and provide a technique to upper bound Wasserstein distance with KL divergence for distributions with Gaussian tail.  
\begin{Lemma}
\label{lemma: I12}
For any $k=1,2,\ldots, N'$, it holds that:
\begin{equation}
\label{eqn: lemma-I12}
\KL\left(\bX_{\tau_k}, \hat{\bX}_{\tau_{k-1}}^{\phi,M}\right)\leqslant \sum_{i=M(k-1)+1}^{Mk} \bE \int_{t_{i-1}}^{t_i} \frac{\beta(t)^2}{2}\left\|s_\phi(\bX_{t_i}, t_i) - \nabla \log p_t(\bX_t)\right\|^2\rd t. 
\end{equation}
Here, the expectation is taken over the forward diffusion process.  Without approximating $\bX_T$ with standard Gaussian distribution $\mathcal N(0, \bm I)$, the forward diffusion and the back diffusion share the trajectory with exactly the same marginal distributions.  
\end{Lemma}
Next, we need to upper bound the right hand side of Equation (\ref{eqn: lemma-I12}). Actually, we can directly use Theorem 9 in \cite{chen2022sampling} and conclude that:
\begin{Lemma}
\label{lemma: expect-KL}
For each $k=1,2,\ldots,N$ and $t\in[t_{k-1}, t_k]$, it holds that:
\[\bE \|s_\phi(\bX_{t_k}, t_k) - \nabla \log p_t(\bX_t)\|^2 \lesssim \varepsilon_{t_k}^2 + L^2 d\Delta t + L^2 \mathcal M_2^2 \Delta t^2 \]
where $\varepsilon_{t_k}^2$ is the score estimation error at time step $t_k$:
\[\varepsilon_{t_k}^2 = \bE_{\bx\sim \bX_{t_k}}\|s_\phi(\bx, t_k) - \nabla \log p_{t_k}(\bx)\|^2,\]
and the expectation is taken over the forward diffusion process. 
\end{Lemma}
After combining Lemma \ref{lemma: I12}, Lemma \ref{lemma: expect-KL} and the score estimation error (\cref{lemma:score}), it holds that: for all $k=1,2,\ldots, N$,
\begin{align}
\label{eqn: I12-1}
\KL\left(\bX_{\tau_k}, \hat{\bX}_{\tau_{k-1}}^{\phi,M}\right)&\leqslant \frac{\overline{\beta}^2}{2} \left(L^2 d\Delta t + L^2 \mathcal M_2^2 \Delta t^2\right)\cdot M\Delta t + \frac{\overline{\beta}^2}{2} \bE \int_{\tau_{k-1}}^{\tau_k} \left\|s_\phi(\bx_t, t)-\nabla\log p_t(\bx_t)\right\|^2 \rd t \notag\\
&\lesssim \overline{\beta}^2\left(L^2d\Delta t \cdot M\Delta t + \frac{1}{\varepsilon} n^{-\frac{2}{d+5}}\cdot M\Delta t\right)
\end{align}
Another major technical result we need is to upper bound Wasserstein distance with KL divergence, which is impossible in the general case. However under Assumption \ref{assump: 1}, we will show that all the variables like $\bX_t$ and $\hat{\bX}_t^{\phi,M}$ have Gaussian tail, which enables the upper bounding. To achieve this, we propose a rigorous notion of Gaussian tail before proving a more general result. 

\begin{Lemma}
\label{lemma: W1-TV}
For constants $c_1, c_2 > 0$,  we call a $d$-dimensional random variable $X$ having a $(c_1, c_2)$-Gaussian tail if there exists a constant $c > 0$ such that
\[\mathbb{P}\left[\|X\|_2 \geqslant t\right] \leqslant  c\cdot \mathbb{P}\left[\|Z\|_2 \geqslant \frac{t - c_1}{c_2}\right]\]
for $\forall t > c_1$ where $Z\sim \mathcal N(0, I_d)$ is a standard Gaussian. Define truncated random variable $X_{R_0}$ as:
\[X_{R_0} = \begin{cases}
X ~~~~~~\mbox{If $\|X\|_2 \leqslant R_0$}\\
0 ~~~~~~~\mbox{If $\|X\|_2 > R_0$}
\end{cases}.\]
Then, the distributional distance $X$ and $X_{R_0}$ is exponentially small with regard to $R_0$ in both Wasserstein and Total Variation metrics:
\[\TV(X, X_{R_0}) \lesssim \exp\left(-\frac{(R_0-c_1)^2}{20c_2^2}\right), ~ W_1(X, X_{R_0}) \lesssim \sqrt{d} c_3^d \cdot \exp\left(-\frac{(R_0-c_1)^2}{40c_2^2}\right)\]
holds for $\forall R_0 > c_1 + \sqrt{2d}\cdot c_2$ where $c_3$ is a constant only dependent on $c_1, c_2$. 
\end{Lemma}
\begin{proof}[Proof of Lemma \ref{lemma: W1-TV}]
Denote $p(x)$ and $p_{R_0}(x)$ as the density function of $X$ and $X_{R_0}$. Then, it is obvious that $p(x) = p_{R_0}(x)$ for $\forall 0 < \|x\|_2 \leqslant R_0$ and $p(x) \geqslant p_{R_0}(x) = 0$ for $\forall \|x\|_2 > R_0$. Another fact is that $p(x) < p_{R_0}(x)$ for $x = 0$. Therefore, the TV-distance between $p(x)$ and $p_{R_0}(x)$ can be simply expressed as:
\[\TV(X, X_{R_0}) = \frac12 \int |p(x)-p_{R_0}(x)| \rd x = \int_{\|x\|\geqslant R_0} p(x) \rd x = \mathbb{P}\left[\|X\|_2 \geqslant R_0\right]. \]
Since we know that $X$ has a $(c_1, c_2)$-Gaussian tail, so:
\[\mathbb{P}\left[\|X\|_2 \geqslant R_0\right] \leqslant c\cdot \mathbb{P}\left[\|Z\|_2 \geqslant \frac{(R_0-c_1)_+}{c_2}\right] = c\cdot \mathbb{P}\left[\|Z\|_2^2 \geqslant \frac{(R_0-c_1)_+^2}{c_2^2}\right].  \]
$\|Z\|_2^2$ follows the $\chi_d^2$ distribution, so for $\forall R_0 > c_1 + c_2\cdot\sqrt{2d}$, its tail bound
\[\TV(X, X_{R_0}) \lesssim \mathbb{P}\left[\|Z\|_2^2 \geqslant \frac{(R_0-c_1)_+^2}{c_2^2}\right] \leqslant \exp\left(-\frac{(R_0-c_1)^2}{20c_2^2}\right). \]
For the Wasserstein-1 distance, we have the following formulation
\begin{equation*}
\begin{aligned}
W_1(X, X_{R_0}) &= \sup_{\substack{\mathrm{Lip}(f)\leqslant 1 \\ f(0)=0}} \int f(x)\cdot \left(p(x) - P_{R_0}(x)\right) \rd x \leqslant \int |f(x)|\cdot |p(x) - P_{R_0}(x)| \rd x\\
&\leqslant  \int \|x\|_2\cdot |p(x) - P_{R_0}(x)| \rd x = \int_{\|x\|>R_0}\|x\|_2\cdot p(x) \rd x \\
& = \mathbb{E}_x \|x\|_2 \cdot \mathbb{I}[\|x\|_2 > R_0] \leqslant \sqrt{\mathbb{E} \|x\|^2} \cdot \sqrt{\mathbb{P}[\|x\|_2 > R_0]} \lesssim \sqrt{d c_3^d} \cdot \exp\left(-\frac{(R_0-c_1)^2}{40c_2^2}\right)
\end{aligned}
\end{equation*}
where $c_3$ is a constant only related to $c_1, c_2$.
\end{proof}

According to Assumption \ref{assump: 1}, we know that the initial distribution $\distri$ has a $(\alpha_1, \alpha_2)$-Gaussian tail. As we move forward, we propose the following properties of the Gaussian tail. 
\begin{Proposition}
\label{prop: tail}
Suppose random variable $X$ has a $(c_1, c_2)$-Gaussian tail, then the following holds:
\begin{itemize}
    \item For any positive constant $c_3 \geqslant c_1$ and $ c_4\geqslant c_2$, it also holds that $X$ has a $(c_3, c_4)$-Gaussian tail. 
    \item For positive constants $a>0$, random variable $aX+b$ has a $(ac_1 + \|b\|_2, ac_2)$-Gaussian tail. 
    \item For $a$-Lipschitz function $F$ with $\|F(0)\|_2 \leqslant b$, then the random variable $F(X)$ has a $(ac_1 + b, ac_2)$-Gaussian tail. 
    \item For a standard Gaussian variable $Y\sim \mathcal N(0,I)$, random variable $aX+bY$ has a $(ac_1, b+ac_2)$-Gaussian tail.   
\end{itemize}
\end{Proposition}
\begin{proof}[Proof of Proposition \ref{prop: tail}] 
According to the definition of Gaussian tail, the first statement is trivial. For $X' := aX + b$ and $\forall t > ac_1 + \|b\|_2$, we have: 
\[\mathbb{P}[\|X'\| \geqslant t] \leqslant \mathbb{P}[a \|X\| \geqslant t-\|b\|]  = c\cdot \mathbb{P}\left[\|X\| \geqslant \frac{t-\|b\|}{a}\right]\]
Since $X$ has a $(c_1, c_2)$-Gaussian tail and $\frac{t-\|b\|}{a} \geqslant c_1$, it holds that:
\[\mathbb{P}\left[\|X\| \geqslant \frac{t-\|b\|}{a}\right] \leqslant \mathbb{P}\left[\|Z\| \geqslant \frac{1}{c_2}\cdot \left(\frac{t-\|b\|}{a}-c_1\right)\right] = \mathbb{P}\left[\|Z\| \geqslant \frac{t-\|b\|-ac_1}{ac_2}\right],\]
which comes to our second statement. For the third statement, we can simply use the result of the second statement since:
\[\|F(X)\|_2 \leqslant \|F(0)\|_2 + a\|X\|_2.\]
In fact, Statement 2 is a special case of Statement 3. For Statement 4, it holds that for $\forall \lambda \in (0,1)$:
\begin{equation*}
\begin{aligned}
\mathbb{P}[\|aX + bY\|_2 \geqslant t] &\leqslant \mathbb{P}\left[\|X\|_2 \geqslant \frac{\lambda t}{a}\right] + \mathbb{P}\left[\|Y\|_2 \geqslant \frac{(1-\lambda)t}{b}\right] \\
&\leqslant c\cdot \mathbb{P}\left[\|Z\|_2 \geqslant \frac{\lambda t-ac_1}{ac_2}\right] + \mathbb{P}\left[\|Y\|_2 \geqslant \frac{(1-\lambda)t}{b}\right]. 
\end{aligned}
\end{equation*}
Let 
\[\frac{\lambda t-ac_1}{ac_2} = \frac{(1-\lambda)t}{b} = \frac{t-ac_1}{b+ac_2},\]
then:
\[\mathbb{P}[\|aX + bY\|_2 \geqslant t]\leqslant (c+1)\cdot \mathbb{P}\left[\|Z\|_2 \geqslant \frac{t-ac_1}{b+ac_2}\right],\]
which means that $aX+bY$ has a $(ac_1, b+ac_2)$-Gaussian tail. 
\end{proof}
Now, for two distributions with Gaussian tail, we show in the next lemma how to upper bound their Wasserstein distance with their total variation distance. 

\begin{Lemma}
\label{lemma: W1-TV-1}
For constants $c_1, c_2, d_1, d_2 > 0 $, for a random variable $X$ with $(c_1, c_2)$-Gaussian tail and another random variable $Y$ with $(d_1, d_2)$-Gaussian tail, we can conclude that:
\[W_1(X, Y) \leqslant C \sqrt{d} \cdot\TV(X, Y) \leqslant C\sqrt{d}\cdot \sqrt{\KL(X,Y)} \]
where $C$ is a constant only dependent on $c_1, c_2, d_1, d_2$. 
\end{Lemma}
\begin{proof}[Proof of Lemma \ref{lemma: W1-TV-1}]
Denote $X_{R_0}, Y_{R_0}$ as the truncated distributions of $X, Y$, then: $X_{R_0}, Y_{R_0}$ has support set $\Omega = \{\bx:~\|\bx\|_2\leqslant R_0\}$. Therefore,
\begin{equation*}
\begin{aligned}
W_1(X_{R_0}, Y_{R_0}) &= \sup_{\substack{\mathrm{Lip}(f)\leqslant 1\\ f(0)=0}} \int f(\bx)(p_{R_0}(\bx)-q_{R_0}(\bx))\rd \bx\leqslant \int \|\bx\|_2\cdot |p_{R_0}(\bx)-q_{R_0}(\bx)|\rd \bx\\
&\leqslant R_0\cdot \int |p_{R_0}(\bx)-q_{R_0}(\bx)|\rd \bx \leqslant 2R_0\cdot \TV(X_{R_0}, Y_{R_0}). 
\end{aligned}
\end{equation*}
Next, we have:
\begin{equation*}
\begin{aligned}
W_1(X, Y) &\leqslant W_1(X, X_{R_0}) + W_1(Y, Y_{R_0}) + W_1(X_{R_0}, Y_{R_0})\\
&\leqslant \sqrt{d} c_3^d \cdot \exp\left(-\frac{(R_0-c_1)^2}{40c_2^2}\right) + \sqrt{d} d_3^d \cdot \exp\left(-\frac{(R_0-d_1)^2}{40d_2^2}\right) + 2R_0\cdot \TV(X_{R_0}, Y_{R_0})\\
&\leqslant \sqrt{d}e_3^d\cdot \exp\left(-\frac{(R_0-e_1)^2}{40e_2^2}\right) + 2R_0\cdot (\TV(X,Y)+\TV(X,X_{R_0})+\TV(Y,Y_{R_0}))\\
&\leqslant \sqrt{d}e_3^d\cdot \exp\left(-\frac{(R_0-e_1)^2}{40e_2^2}\right) + 2R_0\cdot \exp\left(-\frac{(R_0-e_1)^2}{20e_2^2}\right) + 2R_0\cdot \TV(X, Y).
\end{aligned}
\end{equation*}
where $e_i := \max(c_i, d_i)$ for $i=1,2,3$. Let $R_0 = C\sqrt{d}$ for a sufficiently large constant $C$, we can conclude that,
\[W_1(X,Y)\lesssim \sqrt{d}\cdot \TV(X,Y)\]
which comes to our lemma. 
\end{proof}

%% file: I1-I2-bound.tex
According to the optimization constraint, we know that $f_{\hat{\theta}}, f_{\theta^*}\in \mathrm{Lip}(R)$. Therefore, we can combine these two terms $I_1, I_2$ and see how to upper bound 
\[J := \sup_{f_\theta\in\mathrm{Lip}(R)}\sum_{k=1}^{N'} W_1\left(f_\theta(\cdot, \tau_{k-1})_\sharp \bX_{\tau_{k-1}}, f_\theta(\cdot, \tau_{k-1})_\sharp \hat{\bX}_{\tau_{k-1}}^{\phi,M} \right).\] 
Notice that, $\bX_{\tau_{k-1}}, \bX_{\tau_k}$ are sampled from the forward process, which means 
\[\bX_t = (m(t)\cdot\distri) \star \mathcal N(0, \sigma(t)^2)~~~~\forall t\in [\varepsilon,T]\]
where $m(t) = \exp\left(-\int_0^t \beta(s)\rd s\right), \sigma(t)^2 = 1-m(t)^2$ and $\star$ denotes the convolution between two distributions. Also, $\hat{\bX}_{\tau_{k-1}}^{\phi,M}$ is sampled from multi-step discretization of backward probability ODE flow (\cref{eqn:vp-sde-back-ode}), starting from $\bX_{\tau_k}$. \cref{lemma: I12} provides us an upper bound for the KL-divergence between $\bX_{\tau_{k-1}}$ and $\hat{\bX}_{\tau_{k-1}}^{\phi,M}$. Since $f_\theta\in\mathrm{Lip}(R)$, $f_{\theta}(\cdot, \tau_{k-1})$ is an $R$-Lipschitz function, so it holds that:
\[W_1\left(f_\theta(\cdot, \tau_{k-1})_\sharp \bX_{\tau_{k-1}}, f_\theta(\cdot, \tau_{k-1})_\sharp \hat{\bX}_{\tau_{k-1}}^{\phi,M} \right) \leqslant R\cdot W_1\left(\bX_{\tau_{k-1}}, \hat{\bX}_{\tau_{k-1}}^\phi\right).\]
In order to upper bound $W_1\left(\bX_{\tau_{k-1}}, \hat{\bX}_{\tau_{k-1}}^{\phi,M}\right)$ with the KL divergence $\KL\left(\bX_{\tau_{k-1}}, \hat{\bX}_{\tau_{k-1}}^ {\phi,M}\right)$, we need to apply \cref{lemma: W1-TV-1}. In the following part, we prove that the random variable $\bx_{\tau_{k-1}}^{\phi, M}$ has Gaussian tail, just like $\bx_{\tau_k}$. For any integer $k\in [1,N]$, it holds that:
\begin{equation*}
\begin{aligned}
\left\|\hat{\bx}_{t_{k-1}}^\phi\right\|_2 &= \left\|\bx_{t_k} + \left(\frac{\beta(t_k)}{2}\bx_{t_k} + \frac{\beta(t_k)}{2} s_{\phi}(\bx_{t_k}, t)\right)\cdot\Delta t\right\|_2 \\
&< \left(1+\frac{\overline{\beta}\Delta t}{2}\right)\cdot \|\bx_{t_k}\|_2 + \frac{\overline{\beta} U_{\score}\cdot \Delta t}{2} < \left(1+\frac{\Delta t}{U_{\score}}\right) \|\bx_{t_k}\|_2 + \Delta t
\end{aligned}
\end{equation*}
when $\overline{\beta} < 2/U_\score$. After iterating this inequality $M$ times, we have:
\[\left\|\hat{\bx}_{\tau_{k-1}}^{\phi,M}\right\|_2 \leqslant \left(1+\frac{\Delta t}{U_{\score}}\right)^M \|\bx_{\tau_k}\|_2 + \Delta t\cdot \left(1+\left(1+\frac{\Delta t}{U_{\score}}\right)+\ldots + \left(1+\frac{\Delta t}{U_{\score}}\right)^{M-1}\right).\]
Under the condition that $N' \gg T$, we have:
\[\left(1+\frac{\Delta t}{U_\score}\right)^M < (1+\Delta t)^{T/N'\Delta t} < \exp(T/N') < 2,\]
which leads to:
\[\Delta t\cdot \left(1+\left(1+\frac{\Delta t}{U_{\score}}\right)+\ldots + \left(1+\frac{\Delta t}{U_{\score}}\right)^{M-1}\right) \leqslant \Delta t \cdot 2M = \frac{2T}{N'} < 2. \]
Therefore, we have $\left\|\hat{\bx}_{t_{k-1}}^{\phi,M}\right\|_2 \leqslant 2\|\bx_{t_k}\|_2 + 2$, which means both $\bX_{\tau_k}$ and $\bX_{\tau_{k-1}}^{\phi,M}$ have Gaussian tail for all $k\in[1,N']$. By applying \cref{lemma: W1-TV-1}, we conclude that:
\begin{equation}
\label{eqn: I34-final-2}
\begin{aligned}
I_1 + I_2 &\leqslant 2R\cdot \sum_{k=1}^{N'} W_1\left(\bX_{\tau_{k-1}}, \hat{\bX}_{\tau_{k-1}}^{\phi,M}\right) \lesssim 2R\sqrt{d}\cdot\sum_{k=1}^{N'} \sqrt{\KL\left(\bX_{\tau_{k-1}}, \hat{\bX}_{\tau_{k-1}}^{\phi,M}\right)} \\
&\lesssim 2R\sqrt{d} N'\cdot \sqrt{\overline{\beta}^2\left(L^2d\Delta t \cdot M\Delta t + \frac{1}{\varepsilon} n^{-\frac{2}{d+5}}\cdot M\Delta t\right)}\\
&\lesssim RdL\overline{\beta}\cdot \frac{T}{\sqrt{M}} + R\overline{\beta}\sqrt{d} n^{-\frac{1}{d+5}}\cdot \sqrt{\frac{N'T}{\varepsilon}}.
\end{aligned}
\end{equation}
After arranging these terms, we come to our conclusion. 

%% file: I3-I4-bound.tex
For these two loss terms, they can be treated as the gap between empirical and population Wasserstein distances. Notice that, for any two distributions $p, q$, denote $\hat{p}, \hat{q}$ as their empirical version, then it holds that:
\begin{equation}
\label{eqn: wasserstein-triangle}
\left|W_1(p,q)-W_1(\hat{p},\hat{q})\right| \leqslant W_1(p, \hat{p}) + W_1(q, \hat{q}).
\end{equation}
Notice that $f_{\hat{\theta}}(\cdot, \tau_k), f_{\theta^*}(\cdot, \tau_k)$ are Lipschitz-$R$ continuous function for all $k\in [1,N']$. Also, since:
\[G(\bx, t_k;\phi) = \bx + \left(\frac{\beta(t_k)}{2}\bx + \frac{\beta(t_k)}{2} s_{\phi}(\bx, t_k)\right)\cdot\Delta t ,\]
which is Lipschitz continuous with regard to $\bx_{t_k}$ with Lipschitz constant
\[L_1 = 1 + \overline{\beta}(1+L_\score)\Delta t/2 < 1+\Delta t\]
since $\overline{\beta} < 2/(1+L_\score)$. After iterating $M$ times, we know that:
\[G_{(M)}(\cdot, \tau_k;\phi) = G(\cdot, t_{(k-1)M+1}; \phi)\circ \ldots\circ G(\cdot, t_{kM}; \phi)\]
is Lipschitz continuous with constant $(1+\Delta t)^M < \exp(T/N') < 2$. 
Therefore, $f_{\hat{\theta}}(\hat{\bx}_{\tau_{k-1}}^{\phi,M}, \tau_{k-1})$ and $f_{\theta^*}(\hat{\bx}_{\tau_{k-1}}^{\phi,M}, \tau_{k-1})$ are $2R$-Lipschitz continuous function with regard to $\bx_{\tau_k}$. According to Inequality (\ref{eqn: wasserstein-triangle}), we have:
\begin{equation}
\label{eqn: I-56}
\begin{aligned}
|I_3| &\leqslant \sum_{k=1}^{N'} \left[W_1\left(f_{\hat{\theta}}(\cdot, \tau_k)_\sharp \bX_{\tau_k}, f_{\hat{\theta}}(\cdot, t_k)_\sharp \cX_{\tau_k}\right) + W_1\left(f_{\hat{\theta}}(\cdot, \tau_{k-1})_\sharp \hat{\bX}_{\tau_k}^{\phi,M}, f_{\hat{\theta}}(\cdot, \tau_{k-1})_\sharp \hat{\cX}_{\tau_k}^{\phi,M}\right)\right]\\
&\leqslant \sum_{k=1}^{N'} \left[R\cdot W_1\left(\bX_{\tau_k}, \cX_{\tau_k}\right) + R\cdot W_1\left(G_{(M)}(\cdot, \tau_k;\phi)_\sharp \bX_{\tau_k}, G_{(M)}(\cdot, \tau_k;\phi)_\sharp \cX_{\tau_k}\right)\right]\\
&\leqslant \sum_{k=1}^{N'} \left[R\cdot W_1\left(\bX_{\tau_k}, \cX_{\tau_k}\right) + 2R\cdot W_1\left(\bX_{\tau_k}, \cX_{\tau_k}\right)\right] = 3R\cdot \sum_{k=1}^{N'} W_1\left(\bX_{\tau_k}, \cX_{\tau_k}\right). \\
|I_4| &\leqslant \sum_{k=1}^{N'} \left[W_1\left(f_{\theta^*}(\cdot, \tau_k)_\sharp \bX_{\tau_k}, f_{\theta^*}(\cdot, t_k)_\sharp \cX_{\tau_k}\right) + W_1\left(f_{\theta^*}(\cdot, \tau_{k-1})_\sharp \hat{\bX}_{\tau_k}^{\phi,M}, f_{\theta^*}(\cdot, \tau_{k-1})_\sharp \hat{\cX}_{\tau_k}^{\phi,M}\right)\right]\\
&\leqslant \sum_{k=1}^{N'} \left[R\cdot W_1\left(\bX_{\tau_k}, \cX_{\tau_k}\right) + R\cdot W_1\left(G_{(M)}(\cdot, \tau_k;\phi)_\sharp \bX_{\tau_k}, G_{(M)}(\cdot, \tau_k;\phi)_\sharp \cX_{\tau_k}\right)\right]\\
&\leqslant \sum_{k=1}^{N'} \left[R\cdot W_1\left(\bX_{\tau_k}, \cX_{\tau_k}\right) + 2R\cdot W_1\left(\bX_{\tau_k}, \cX_{\tau_k}\right)\right] = 3R\cdot \sum_{k=1}^{N'} W_1\left(\bX_{\tau_k}, \cX_{\tau_k}\right).
\end{aligned}
\end{equation}
Here, $\cX_{\tau_k}$ is the empirical version of distribution $\bX_{\tau_k}$, which means:
\[\bX_{\tau_k} = (m(\tau_k)\cdot \distri)\star \cN(0, \sigma(\tau_k)^2), ~\cX_{\tau_k} = (m(\tau_k)\cdot \hat{\distri})\star \cN(0, \sigma(\tau_k)^2)\]
where $\hat{\distri}$ is a uniform distribution taken over the $n$ i.i.d samples from $\distri$. In order to upper bound $|I_3|, |I_4|$, we only need to control $W_1\left(\bX_{\tau_k}, \cX_{\tau_k}\right)$. The following lemma upper bounds $W_1\left(\bX_{\tau_k}, \cX_{\tau_k}\right)$ with $W_1\left(\distri, \hat{\distri}\right)$ for each $k=1,2,\ldots, N'$. 
\begin{Lemma}
\label{lemma: others-1}
For each $k=1,2,\ldots, N'$, it holds that:
\[W_1\left(\bX_{\tau_k}, \cX_{\tau_k}\right) \leqslant m(\tau_k)\cdot W_1\left(\distri, ~\hat{\distri}\right). \]
\end{Lemma}
\begin{proof}
By the dual formulation of Wasserstein distance, it holds that for $\forall t\in[0,T]$:
\[W_1(\bX_t, \cX_t) = \sup_{\mathrm{Lip}(F)\leqslant 1} \left(\bE_{\bx\sim \bX_t} F(\bx) - \bE_{\bx\sim \cX_t} F(\hat{\bx})\right) = \sup_{\mathrm{Lip}(F)\leqslant 1} \Big[\bE_{\bx}\bE_{z} F(m_t \bx + \sigma_t z) - \bE_{\hat{\bx}}\bE_{z} F(m_t \hat{\bx} + \sigma_t z)\Big] \]
where the expectation is taken over $\bx\sim\distri, \hat{\bx}\sim \hat{\distri}$ and $z\sim \mathcal N(0,\bm I)$. Notice that, for the following mapping
\[G[F](\bx) := \bE_z F(m(t) \bx + \sigma(t) z),\]
it holds that for any function $F$ with Lipschitz constant 1, 
\[|G[F](\bx) - G[F](\bx')| \leqslant \bE_z \left|F(m(t) \bx + \sigma(t) z) - F(m(t) \bx' + \sigma(t) z)\right| \leqslant \bE_z \left[m(t)\cdot |\bx-\bx'|\right] = m(t)\cdot |\bx-\bx'|. \]
Therefore, we know that $G[F]$ is $m(t)$-Lipschitz continuous, which leads to
\begin{equation*}
\begin{aligned}
W_1(\bX_t, \hat{\bX_t})&=\sup_{\mathrm{Lip}(F)\leqslant 1} \Big[\bE_{\bx} G[F](\bx) - \bE_{\hat{\bx}} G[F](\hat{\bx})\Big] \leqslant \sup_{\mathrm{Lip}(G)\leqslant m_t} \Big[\bE_{\bx} G(\bx) - \bE_{\hat{\bx}} G(\hat{\bx})\Big]\\
&= m(t)\cdot \sup_{\mathrm{Lip}(G)\leqslant 1} \Big[\bE_{\bx} G(\bx) - \bE_{\hat{\bx}} G(\hat{\bx})\Big] = m(t)\cdot W_1(\distri, ~\hat{\distri}).
\end{aligned}
\end{equation*}
It comes to our conclusion.
\end{proof}
After that, by combining Equation (\ref{eqn: I-56}) and Lemma \ref{lemma: others-1}, we have:
\begin{equation}
\label{eqn: I-56-2}
|I_3| + |I_4| \leqslant 6R\left(\sum_{k=1}^{N'} m(\tau_k)\right) \cdot W_1\left(\distri, \hat{\distri}\right) < 6RN'\cdot W_1\left(\distri, \hat{\distri}\right). 
\end{equation}
Now, our final step is to bound the gap between empirical and population Wasserstein distance of the initial distribution $\distri$. According to the statistical result \cite{weed2019sharp}, we can conclude that:
\[\bE W_1\left(\distri, \hat{p}_{\mathrm{data}}\right) \lesssim n^{-1/d}.\]
Here, the expectation is taken over $\hat{p}_{\mathrm{data}}\overset{i.i.d}{\sim} \distri$. 
Therefore, we can upper bound $I_3 + I_4 $ as:
\begin{equation}
\label{eqn: final-I56}
I_3 + I_4 \leqslant |I_3| + |I_4| \lesssim 6RN'\cdot n^{-1/d}. 
\end{equation}

%% file: decomposition-isolation.tex
Unlike the distillation case, the consistency equality we apply is based on the empirical distributions, so that $\theta^*$ satisfies
\[f_{\theta^*}(\cdot, \tau_k)_\sharp \cX_{\tau_k} \overset{\rm law}{=} f_{\theta^*}(\cdot, \tau_{k-1})_\sharp \hat{\cX}_{\tau_{k-1}}^M~~~\forall k \in[N'] \]
because of the definition of $f_{\theta^*}$.  Besides, according to the optimality inequality, we have:
\[\sum_{k=1}^{N'} W_1 \left(f_{\hat{\theta}}(\cdot, \tau_k)_\sharp \cX_{\tau_k}, f_{\hat{\theta}}(\cdot, \tau_{k-1})_\sharp \cX_{\tau_{k-1}}\right) \leqslant \sum_{k=1}^{N'} W_1 \left(f_{\theta^*}(\cdot, \tau_k)_\sharp \cX_{\tau_k}, f_{\theta^*}(\cdot, \tau_{k-1})_\sharp \cX_{\tau_{k-1}}\right). \]
After combining these two inequalities, we can upper bound our target function as follows:
\begin{Lemma}
\label{lemma:decompose-isolation}
\[W_1\left(f_{\hat{\theta}}(\cdot, T)_\sharp \cX_T, f_{\theta^*}(\cdot, T)_\sharp \cX_T\right) \leqslant 2 \sum_{k=1}^{N'} W_1\left(f_{\theta^*}(\cdot, \tau_{k-1})_\sharp \hat{\cX}_{\tau_{k-1}}^{M}, f_{\theta^*}(\cdot, \tau_{k-1})_\sharp \cX_{\tau_{k-1}}\right). \]
\end{Lemma}
\begin{proof}
Notice that for $\forall k\in [N']$, we have:
\begin{equation*}
\begin{aligned}
&W_1\left(f_{\hat{\theta}}(\cdot, \tau_k)_\sharp \cX_{\tau_k}, f_{\theta^*}(\cdot, \tau_k)_\sharp \cX_{\tau_k}\right)\leqslant W_1\left(f_{\hat{\theta}}(\cdot, \tau_{k-1})_\sharp \cX_{\tau_{k-1}}, f_{\theta^*}(\cdot, \tau_{k-1})_\sharp \cX_{\tau_{k-1}}\right)\\
&~~~+ W_1\left(f_{\hat{\theta}}(\cdot, \tau_k)_\sharp \cX_{\tau_k}, f_{\hat{\theta}}(\cdot, \tau_{k-1})_\sharp \cX_{\tau_{k-1}}\right)+ W_1\left(f_{\theta^*}(\cdot, \tau_k)_\sharp \cX_{\tau_k}, f_{\theta^*}(\cdot, \tau_{k-1})_\sharp \cX_{\tau_{k-1}}\right).
\end{aligned}
\end{equation*}
After taking summation over $k=1,2,\ldots, N'$, we have:
\begin{equation*}
\begin{aligned}
&~~~W_1\left(f_{\hat{\theta}}(\cdot, T)_\sharp \cX_T, f_{\theta^*}(\cdot, T)_\sharp \cX_T\right)\\
&\leqslant \sum_{k=1}^{N'} W_1\left(f_{\hat{\theta}}(\cdot, \tau_k)_\sharp \cX_{\tau_k}, f_{\hat{\theta}}(\cdot, \tau_{k-1})_\sharp \cX_{\tau_{k-1}}\right)+ \sum_{k=1}^{N'} W_1\left(f_{\theta^*}(\cdot, \tau_k)_\sharp \cX_{\tau_k}, f_{\theta^*}(\cdot, \tau_{k-1})_\sharp \cX_{\tau_{k-1}}\right)\\
&\leqslant \sum_{k=1}^{N'} W_1\left(f_{\theta^*}(\cdot, \tau_k)_\sharp \cX_{\tau_k}, f_{\theta^*}(\cdot, \tau_{k-1})_\sharp \cX_{\tau_{k-1}}\right)+\sum_{k=1}^{N'} W_1\left(f_{\theta^*}(\cdot, \tau_k)_\sharp \cX_{\tau_k}, f_{\theta^*}(\cdot, \tau_{k-1})_\sharp \cX_{\tau_{k-1}}\right)\\
&= 2\sum_{k=1}^{N'} W_1\left(f_{\theta^*}(\cdot, \tau_{k-1})_\sharp \hat{\cX}_{\tau_{k-1}}^{M}, f_{\theta^*}(\cdot, \tau_{k-1})_\sharp \cX_{\tau_{k-1}}\right),
\end{aligned}
\end{equation*}
which comes to our conclusion. Here, we use the optimality inequality as well as the consistency equation.
\end{proof}
As we can see, the loss decomposition is much simpler than the distillation case. The only relevant term stands for the discretization error of ODE solver. Since $f_{\theta^*}(\cdot, t)$ is $R$-Lipschitz for any $t\in [0,1]$, we have
\[W_1\left(f_{\theta^*}(\cdot, \tau_{k-1})_\sharp \hat{\cX}_{\tau_{k-1}}^{M}, f_{\theta^*}(\cdot, \tau_{k-1})_\sharp \cX_{\tau_{k-1}}\right) \leqslant R \cdot W_1\left(\hat{\cX}_{\tau_{k-1}}^{M}, \cX_{\tau_{k-1}}\right). \]
Compared with \cref{eqn: I12-1}, we do not have the score approximation error here since the score function for $\cX_t$ has explicit formulation, which leads to
\[\KL\left(\hat{\cX}_{\tau_{k-1}}^{M}, \cX_{\tau_{k-1}}\right)  \lesssim \overline{\beta}^2 L_\varepsilon^2 d\Delta t\cdot  M \Delta t. \]
Here, the score function $\nabla \log \hat{p}_t(\cdot)$ is $L_\varepsilon$-Lipschitz continuous for $\forall t\in [\varepsilon, T]$. By using \cref{lemma: W1-TV-1} and \cref{lemma:decompose-isolation}, we have:
\begin{equation*}
W_1\left(f_{\hat{\theta}}(\cdot, T)_\sharp \cX_T, f_{\theta^*}(\cdot, T)_\sharp \cX_T\right) \leqslant 2R \cdot \sum_{k=1}^{N'} W_1\left(\hat{\cX}_{\tau_{k-1}}^M, \cX_{\tau_{k-1}}\right) \lesssim 2R N'\sqrt{d}\cdot \sqrt{\overline{\beta}^2 L_\varepsilon^2 dM \Delta t^2}.
\end{equation*}
Similarly, the DDPM bound (Lemma \ref{lemma:ddpm}) also does not contain the score approximation error:
\[\KL\left(f_{\theta^*}(\cdot, T)_\sharp \cX_T, \cX_\varepsilon\right)\lesssim \overline{\beta}^2 L_\varepsilon^2 Td\Delta t, \]
which leads to 
\[W_1\left(f_{\theta^*}(\cdot, T)_\sharp \cX_T, \cX_\varepsilon\right)\lesssim \sqrt{d}\cdot \sqrt{\overline{\beta}^2 L_\varepsilon^2 Td\Delta t}\]
according to Lemma \ref{lemma: W1-TV-1}. In the next step, we need to bound the Lipschitz constant $L_\varepsilon$ since Assumption \ref{assump: 1} is no longer applicable here. 
\begin{Lemma}
\label{lemma:isolation-lipschitz}
For the mixture of Gaussian distribution $\frac1n\sum_{j=1}^n \mathcal N(\bx^j, \sigma^2\bm I)$, we denote $\hat{p}$ as its density. Assume $\|\bx^j\|_2\leqslant R_0$ for $\forall j\in [n]$, then its score function $\nabla \log \hat{p}(\cdot)$ is $L$-Lipschitz continuous. Here $L = \max(R_0^2/\sigma^4, 1/\sigma^2)$. Furthermore, it leads to $L_\varepsilon = 4R_0^2/(\varepsilon^2\underline{\beta}^2)$. 
\end{Lemma}
\begin{proof}
For the score function of $\frac1n\sum_{j=1}^n \mathcal N(\bx^j, \sigma^2)$, it has the following formulation:
\[s(\bx) = \frac{\sum_{j=1}^n -\frac{\bx-\bx^j}{\sigma^2}\exp\left(-\frac{\|\bx-\bx^j\|^2}{2\sigma^2}\right)}{\sum_{j=1}^n \exp\left(-\frac{\|\bx-\bx^j\|^2}{2\sigma^2}\right)} = -\sum_{j=1}^n \frac{\bx-\bx^j}{\sigma^2}\cdot p^j = -\frac{1}{\sigma^2}\bx + \frac{1}{\sigma^2}\sum_{j=1}^n p^j \bx^j.\]
Here, 
\[p^j = \frac{\exp\left(-\frac{\|\bx-\bx^j\|^2}{2\sigma^2}\right)}{\sum_{k=1}^n\exp\left(-\frac{\|\bx-\bx^k\|^2}{2\sigma^2}\right)}~~~~\forall j\in [n].\]
The Jacobian matrix
\begin{equation}
\label{eqn:Jacobian}
\begin{aligned}
\frac{\rd s(\bx)}{\rd \bx} &= -\frac{1}{\sigma^2}\mathrm{Id} + \frac{1}{\sigma^2} \sum_{j=1}^n \bx^j \cdot \left(\frac{\rd p^j}{\rd \bx}\right)^\top = -\frac{1}{\sigma^2}\mathrm{Id} - \frac{1}{\sigma^2} \sum_{j=1}^n \bx^j \cdot \left(p^j\cdot \frac{\bx-\bx^j}{\sigma^2} + p^j s(\bx)\right)^\top \\
&= -\frac{1}{\sigma^2}\mathrm{Id} + \frac{1}{\sigma^4}\sum_{j=1}^n p^j \bx^j \bx^{j\top} - \frac{1}{\sigma^2} \left(\sum_{j=1}^n p^j\bx^j\right)\cdot \left(s(\bx)+\frac{1}{\sigma^2}x\right)^\top\\
&= -\frac{1}{\sigma^2}\mathrm{Id} + \frac{1}{\sigma^4}\sum_{j=1}^n p^j \bx^j \bx^{j\top} - \frac{1}{\sigma^4} \left(\sum_{j=1}^n p^j\bx^j\right)\cdot \left(\sum_{j=1}^n p^j\bx^j\right)^\top.
\end{aligned}
\end{equation}
Therefore, we have:
\[-\frac{1}{\sigma^2}\mathrm{Id} \preceq \frac{\rd s(\bx)}{\rd \bx} \preceq \frac{1}{\sigma^4} \sum_{j=1}^n p^j \bx^j \bx^{j\top}.\]
Notice that when $\|\bx^j\|_2 \leqslant R_0$ for all $j\in [n]$:
\[\left\|\sum_{j=1}^n p^j \bx^j \bx^{j\top}\right\|_2 \leqslant \sum_{j=1}^n p^j \left\|\bx^j \bx^{j\top}\right\|_2 = \sum_{j=1}^n p^j \|\bx^j\|_2^2 \leqslant R_0^2. \]
Finally, we can conclude that $s(\bx)$ is $L$-Lipschitz continuous for $L = \max(1/\sigma^2, R_0^2/\sigma^4)$. Furthermore, for $L_\varepsilon$, the Lipschitz continuity of $\nabla \log \hat{p}_t(\cdot)$ for $t\in [\varepsilon, T]$, since 
\[\sigma(\varepsilon)^2 = 1-m(\varepsilon)^2 \geqslant 1-\exp(\underline{\beta}\varepsilon) > \underline{\beta}\varepsilon /2,\]
we have $L_\varepsilon = 4R_0^2/(\underline{\beta}^2 \varepsilon^2)$. 
\end{proof}
After combining these conclusions together, we notice that:
\begin{equation*}
\begin{aligned}
W_1\left(f_{\hat{\theta}}(\cdot, T)_\sharp \mathcal N(0, \bm I), \distri\right) &\leqslant W_1\left(f_{\hat{\theta}}(\cdot, T)_\sharp \mathcal N(0, \bm I), f_{\hat{\theta}}(\cdot, T)_\sharp\cX_T\right) + W_1\left(f_{\hat{\theta}}(\cdot, T)_\sharp \cX_T, f_{\theta^*}(\cdot, T)_\sharp\cX_T\right)\\
&~~+ W_1\left(f_{\theta^*}(\cdot, T)_\sharp\cX_T, \cX_\varepsilon\right) + W_1(\cX_\varepsilon, \distri)\\
&\lesssim R\cdot W_1(\mathcal N(0, \bm I), \cX_T) +  2R N'\sqrt{d}\cdot \sqrt{\overline{\beta}^2 L_\varepsilon^2 dM \Delta t^2} \\
&~~+ \sqrt{d}\cdot \sqrt{\overline{\beta}^2 L_\varepsilon^2 Td\Delta t} + W_1(\cX_\varepsilon, X_\varepsilon) + W_1(X_\varepsilon, \distri). 
\end{aligned}
\end{equation*}
Finally, we apply Lemma \ref{lemma: others-1}, \ref{lemma:gaussian}, \ref{lemma:distri-varepsilon}, and have:
\[W_1(X_\varepsilon, \distri)\lesssim \sqrt{d\overline{\beta}\varepsilon}, ~~\bE\left[W_1(\cX_\varepsilon, X_\varepsilon) \right]\leqslant \bE\left[W_1(\hat{\distri}, \distri)\right]\lesssim n^{-1/d}\] 
\[W_1(\mathcal N(0, \bm I), \cX_T)\leqslant W_1(\mathcal N(0, \bm I), X_T) + W_1(X_T, \cX_T)\lesssim \sqrt{d}\exp(-\underline{\beta}T/2) + n^{-1/d}. \]
To sum up, it holds that:
\begin{equation}
\label{eqn:thm-2}
\begin{aligned}
\bE\left[W_1\left(f_{\hat{\theta}}(\cdot, T)_\sharp \mathcal N(0, \bm I), \distri\right)\right]&\lesssim \sqrt{d} R \exp\left(-\underline{\beta}T/2\right) + R \cdot n^{-1/d} + d\overline{\beta} L_\varepsilon \cdot \frac{T}{\sqrt{M}} + \sqrt{d\overline{\beta}\varepsilon}\\
&\lesssim \sqrt{d} R \exp\left(-\underline{\beta}T/2\right) + R \cdot n^{-1/d} + \frac{d\overline{\beta}R_0^2}{\underline{\beta}^2 \varepsilon^2}  \cdot \frac{T}{\sqrt{M}} + \sqrt{d\overline{\beta}\varepsilon},
\end{aligned}
\end{equation}
which comes to our conclusion of the main theorem \ref{thm:isolation}.

%% file: main.bbl
\newcommand{\etalchar}[1]{$^{#1}$}
\begin{thebibliography}{GPAM{\etalchar{+}}20}

\bibitem[ABVE23]{albergo2023stochastic}
Michael~S Albergo, Nicholas~M Boffi, and Eric Vanden-Eijnden.
\newblock Stochastic interpolants: A unifying framework for flows and diffusions.
\newblock {\em arXiv preprint arXiv:2303.08797}, 2023.

\bibitem[And82]{anderson1982reverse}
Brian~DO Anderson.
\newblock Reverse-time diffusion equation models.
\newblock {\em Stochastic Processes and their Applications}, 12(3):313--326, 1982.

\bibitem[BDBDD23]{benton2023linear}
Joe Benton, Valentin De~Bortoli, Arnaud Doucet, and George Deligiannidis.
\newblock Linear convergence bounds for diffusion models via stochastic localization.
\newblock {\em arXiv preprint arXiv:2308.03686}, 2023.

\bibitem[BMR20]{block2020generative}
Adam Block, Youssef Mroueh, and Alexander Rakhlin.
\newblock Generative modeling with denoising auto-encoders and langevin sampling.
\newblock {\em arXiv preprint arXiv:2002.00107}, 2020.

\bibitem[Caf92]{caffarelli1992regularity}
Luis~A Caffarelli.
\newblock The regularity of mappings with a convex potential.
\newblock {\em Journal of the American Mathematical Society}, 5(1):99--104, 1992.

\bibitem[CCL{\etalchar{+}}22]{chen2022sampling}
Sitan Chen, Sinho Chewi, Jerry Li, Yuanzhi Li, Adil Salim, and Anru~R Zhang.
\newblock Sampling is as easy as learning the score: theory for diffusion models with minimal data assumptions.
\newblock {\em arXiv preprint arXiv:2209.11215}, 2022.

\bibitem[CCL{\etalchar{+}}23]{chen2023probability}
Sitan Chen, Sinho Chewi, Holden Lee, Yuanzhi Li, Jianfeng Lu, and Adil Salim.
\newblock The probability flow ode is provably fast.
\newblock {\em arXiv preprint arXiv:2305.11798}, 2023.

\bibitem[CDD23]{chen2023restoration}
Sitan Chen, Giannis Daras, and Alex Dimakis.
\newblock Restoration-degradation beyond linear diffusions: A non-asymptotic analysis for ddim-type samplers.
\newblock In {\em International Conference on Machine Learning}, pages 4462--4484. PMLR, 2023.

\bibitem[CFD{\etalchar{+}}23]{chi2023diffusion}
Cheng Chi, Siyuan Feng, Yilun Du, Zhenjia Xu, Eric Cousineau, Benjamin Burchfiel, and Shuran Song.
\newblock Diffusion {P}olicy: Visuomotor policy learning via action diffusion.
\newblock {\em arXiv preprint arXiv:2303.04137}, 2023.

\bibitem[CHZW23]{chen2023score}
Minshuo Chen, Kaixuan Huang, Tuo Zhao, and Mengdi Wang.
\newblock Score approximation, estimation and distribution recovery of diffusion models on low-dimensional data.
\newblock {\em arXiv preprint arXiv:2302.07194}, 2023.

\bibitem[CZZ{\etalchar{+}}20]{chen2020wavegrad}
Nanxin Chen, Yu~Zhang, Heiga Zen, Ron~J Weiss, Mohammad Norouzi, and William Chan.
\newblock Wavegrad: Estimating gradients for waveform generation.
\newblock {\em arXiv preprint arXiv:2009.00713}, 2020.

\bibitem[DB22]{de2022convergence}
Valentin De~Bortoli.
\newblock Convergence of denoising diffusion models under the manifold hypothesis.
\newblock {\em arXiv preprint arXiv:2208.05314}, 2022.

\bibitem[DBTHD21]{de2021diffusion}
Valentin De~Bortoli, James Thornton, Jeremy Heng, and Arnaud Doucet.
\newblock Diffusion schr{\"o}dinger bridge with applications to score-based generative modeling.
\newblock {\em Advances in Neural Information Processing Systems}, 34:17695--17709, 2021.

\bibitem[DML{\etalchar{+}}19]{dathathri2019plug}
Sumanth Dathathri, Andrea Madotto, Janice Lan, Jane Hung, Eric Frank, Piero Molino, Jason Yosinski, and Rosanne Liu.
\newblock Plug and play language models: A simple approach to controlled text generation.
\newblock {\em arXiv preprint arXiv:1912.02164}, 2019.

\bibitem[DN21]{dhariwal2021diffusion}
Prafulla Dhariwal and Alexander Nichol.
\newblock Diffusion models beat gans on image synthesis.
\newblock {\em Advances in Neural Information Processing Systems}, 34:8780--8794, 2021.

\bibitem[EAMS23]{el2023sampling}
Ahmed El~Alaoui, Andrea Montanari, and Mark Sellke.
\newblock Sampling from mean-field gibbs measures via diffusion processes.
\newblock {\em arXiv preprint arXiv:2310.08912}, 2023.

\bibitem[GPAM{\etalchar{+}}20]{goodfellow2020generative}
Ian Goodfellow, Jean Pouget-Abadie, Mehdi Mirza, Bing Xu, David Warde-Farley, Sherjil Ozair, Aaron Courville, and Yoshua Bengio.
\newblock Generative adversarial networks.
\newblock {\em Communications of the ACM}, 63(11):139--144, 2020.

\bibitem[GSF{\etalchar{+}}23]{gruver2023protein}
Nate Gruver, Samuel Stanton, Nathan~C Frey, Tim~GJ Rudner, Isidro Hotzel, Julien Lafrance-Vanasse, Arvind Rajpal, Kyunghyun Cho, and Andrew~Gordon Wilson.
\newblock Protein design with guided discrete diffusion.
\newblock {\em arXiv preprint arXiv:2305.20009}, 2023.

\bibitem[HEKJ{\etalchar{+}}23]{hansen2023idql}
Philippe Hansen-Estruch, Ilya Kostrikov, Michael Janner, Jakub~Grudzien Kuba, and Sergey Levine.
\newblock {IDQL}: Implicit {Q}-learning as an actor-critic method with diffusion policies.
\newblock {\em arXiv preprint arXiv:2304.10573}, 2023.

\bibitem[HJA20]{ho2020denoising}
Jonathan Ho, Ajay Jain, and Pieter Abbeel.
\newblock Denoising diffusion probabilistic models.
\newblock {\em Advances in Neural Information Processing Systems}, 33:6840--6851, 2020.

\bibitem[HZRS16]{he2016deep}
Kaiming He, Xiangyu Zhang, Shaoqing Ren, and Jian Sun.
\newblock Deep residual learning for image recognition.
\newblock In {\em Proceedings of the IEEE conference on computer vision and pattern recognition}, pages 770--778, 2016.

\bibitem[KAAL22]{karras2022elucidating}
Tero Karras, Miika Aittala, Timo Aila, and Samuli Laine.
\newblock Elucidating the design space of diffusion-based generative models.
\newblock {\em Advances in Neural Information Processing Systems}, 35:26565--26577, 2022.

\bibitem[KPH{\etalchar{+}}20]{kong2020diffwave}
Zhifeng Kong, Wei Ping, Jiaji Huang, Kexin Zhao, and Bryan Catanzaro.
\newblock Diffwave: A versatile diffusion model for audio synthesis.
\newblock {\em arXiv preprint arXiv:2009.09761}, 2020.

\bibitem[KW{\etalchar{+}}19]{kingma2019introduction}
Diederik~P Kingma, Max Welling, et~al.
\newblock An introduction to variational autoencoders.
\newblock {\em Foundations and Trends{\textregistered} in Machine Learning}, 12(4):307--392, 2019.

\bibitem[LCF23]{lyu2023convergence}
Junlong Lyu, Zhitang Chen, and Shoubo Feng.
\newblock Convergence guarantee for consistency models.
\newblock {\em arXiv preprint arXiv:2308.11449}, 2023.

\bibitem[LKK22]{lee2022proteinsgm}
Jin~Sub Lee, Jisun Kim, and Philip~M Kim.
\newblock Proteinsgm: Score-based generative modeling for de novo protein design.
\newblock {\em bioRxiv}, pages 2022--07, 2022.

\bibitem[LKW{\etalchar{+}}22]{lovelace2022latent}
Justin Lovelace, Varsha Kishore, Chao Wan, Eliot Shekhtman, and Kilian Weinberger.
\newblock Latent diffusion for language generation.
\newblock {\em arXiv preprint arXiv:2212.09462}, 2022.

\bibitem[LLT22a]{lee2022convergencea}
Holden Lee, Jianfeng Lu, and Yixin Tan.
\newblock Convergence for score-based generative modeling with polynomial complexity.
\newblock {\em arXiv preprint arXiv:2206.06227}, 2022.

\bibitem[LLT22b]{lee2022convergenceb}
Holden Lee, Jianfeng Lu, and Yixin Tan.
\newblock Convergence of score-based generative modeling for general data distributions.
\newblock {\em arXiv preprint arXiv:2209.12381}, 2022.

\bibitem[LRJ{\etalchar{+}}23]{li2023diffusion}
Xin Li, Yulin Ren, Xin Jin, Cuiling Lan, Xingrui Wang, Wenjun Zeng, Xinchao Wang, and Zhibo Chen.
\newblock Diffusion models for image restoration and enhancement--a comprehensive survey.
\newblock {\em arXiv preprint arXiv:2308.09388}, 2023.

\bibitem[LSP{\etalchar{+}}22]{luo2022antigen}
Shitong Luo, Yufeng Su, Xingang Peng, Sheng Wang, Jian Peng, and Jianzhu Ma.
\newblock Antigen-specific antibody design and optimization with diffusion-based generative models for protein structures.
\newblock {\em Advances in Neural Information Processing Systems}, 35:9754--9767, 2022.

\bibitem[LTG{\etalchar{+}}22]{li2022diffusion}
Xiang Li, John Thickstun, Ishaan Gulrajani, Percy~S Liang, and Tatsunori~B Hashimoto.
\newblock Diffusion-lm improves controllable text generation.
\newblock {\em Advances in Neural Information Processing Systems}, 35:4328--4343, 2022.

\bibitem[LWYL22]{liu2022let}
Xingchao Liu, Lemeng Wu, Mao Ye, and Qiang Liu.
\newblock Let us build bridges: Understanding and extending diffusion generative models.
\newblock {\em arXiv preprint arXiv:2208.14699}, 2022.

\bibitem[LZB{\etalchar{+}}22]{lu2022dpm}
Cheng Lu, Yuhao Zhou, Fan Bao, Jianfei Chen, Chongxuan Li, and Jun Zhu.
\newblock Dpm-solver: A fast ode solver for diffusion probabilistic model sampling in around 10 steps.
\newblock {\em Advances in Neural Information Processing Systems}, 35:5775--5787, 2022.

\bibitem[MW23a]{mei2023deep}
Song Mei and Yuchen Wu.
\newblock Deep networks as denoising algorithms: Sample-efficient learning of diffusion models in high-dimensional graphical models.
\newblock {\em arXiv preprint arXiv:2309.11420}, 2023.

\bibitem[MW23b]{montanari2023posterior}
Andrea Montanari and Yuchen Wu.
\newblock Posterior sampling from the spiked models via diffusion processes.
\newblock {\em arXiv preprint arXiv:2304.11449}, 2023.

\bibitem[ND21]{nichol2021improved}
Alexander~Quinn Nichol and Prafulla Dhariwal.
\newblock Improved denoising diffusion probabilistic models.
\newblock In {\em Proceedings of the International Conference on Machine Learning}, pages 8162--8171. PMLR, 2021.

\bibitem[OAS23]{oko2023diffusion}
Kazusato Oko, Shunta Akiyama, and Taiji Suzuki.
\newblock Diffusion models are minimax optimal distribution estimators.
\newblock {\em arXiv preprint arXiv:2303.01861}, 2023.

\bibitem[PRK{\etalchar{+}}23]{pearce2023imitating}
Tim Pearce, Tabish Rashid, Anssi Kanervisto, Dave Bignell, Mingfei Sun, Raluca Georgescu, Sergio~Valcarcel Macua, Shan~Zheng Tan, Ida Momennejad, Katja Hofmann, and Sam Devlin.
\newblock Imitating human behaviour with diffusion models.
\newblock {\em arXiv preprint arXiv:2301.10677}, 2023.

\bibitem[RBL{\etalchar{+}}22]{rombach2022high}
Robin Rombach, Andreas Blattmann, Dominik Lorenz, Patrick Esser, and Bj{\"o}rn Ommer.
\newblock High-resolution image synthesis with latent diffusion models.
\newblock In {\em Proceedings of the IEEE/CVF Conference on Computer Vision and Pattern Recognition}, pages 10684--10695, 2022.

\bibitem[RLJL23]{reuss2023goal}
Moritz Reuss, Maximilian Li, Xiaogang Jia, and Rudolf Lioutikov.
\newblock Goal-conditioned imitation learning using score-based diffusion policies.
\newblock {\em arXiv preprint arXiv:2304.02532}, 2023.

\bibitem[SDCS23]{song2023consistency}
Yang Song, Prafulla Dhariwal, Mark Chen, and Ilya Sutskever.
\newblock Consistency models.
\newblock {\em arXiv preprint arXiv:2303.01469}, 2023.

\bibitem[SE19]{song2019generative}
Yang Song and Stefano Ermon.
\newblock Generative modeling by estimating gradients of the data distribution.
\newblock {\em Advances in Neural Information Processing Systems}, 32, 2019.

\bibitem[SE20]{song2020improved}
Yang Song and Stefano Ermon.
\newblock Improved techniques for training score-based generative models.
\newblock {\em Advances in neural information processing systems}, 33:12438--12448, 2020.

\bibitem[SGSE20]{song2020sliced}
Yang Song, Sahaj Garg, Jiaxin Shi, and Stefano Ermon.
\newblock Sliced score matching: A scalable approach to density and score estimation.
\newblock In {\em Uncertainty in Artificial Intelligence}, pages 574--584. PMLR, 2020.

\bibitem[SME20]{song2020denoising}
Jiaming Song, Chenlin Meng, and Stefano Ermon.
\newblock Denoising diffusion implicit models.
\newblock {\em arXiv preprint arXiv:2010.02502}, 2020.

\bibitem[SSDK{\etalchar{+}}20]{song2020score}
Yang Song, Jascha Sohl-Dickstein, Diederik~P Kingma, Abhishek Kumar, Stefano Ermon, and Ben Poole.
\newblock Score-based generative modeling through stochastic differential equations.
\newblock {\em arXiv preprint arXiv:2011.13456}, 2020.

\bibitem[Vin11]{vincent2011connection}
Pascal Vincent.
\newblock A connection between score matching and denoising autoencoders.
\newblock {\em Neural computation}, 23(7):1661--1674, 2011.

\bibitem[Wai19]{wainwright2019high}
Martin~J Wainwright.
\newblock {\em High-dimensional statistics: A non-asymptotic viewpoint}, volume~48.
\newblock Cambridge university press, 2019.

\bibitem[WB19]{weed2019sharp}
Jonathan Weed and Francis Bach.
\newblock Sharp asymptotic and finite-sample rates of convergence of empirical measures in wasserstein distance.
\newblock 2019.

\bibitem[YHN{\etalchar{+}}23]{yuan2023reward}
Hui Yuan, Kaixuan Huang, Chengzhuo Ni, Minshuo Chen, and Mengdi Wang.
\newblock Reward-directed conditional diffusion: Provable distribution estimation and reward improvement.
\newblock {\em arXiv preprint arXiv:2307.07055}, 2023.

\bibitem[YHN{\etalchar{+}}24]{yuan2024reward}
Hui Yuan, Kaixuan Huang, Chengzhuo Ni, Minshuo Chen, and Mengdi Wang.
\newblock Reward-directed conditional diffusion: Provable distribution estimation and reward improvement.
\newblock {\em Advances in Neural Information Processing Systems}, 36, 2024.

\bibitem[YXM{\etalchar{+}}22]{yu2022latent}
Peiyu Yu, Sirui Xie, Xiaojian Ma, Baoxiong Jia, Bo~Pang, Ruiqi Gao, Yixin Zhu, Song-Chun Zhu, and Ying~Nian Wu.
\newblock Latent diffusion energy-based model for interpretable text modeling.
\newblock {\em arXiv preprint arXiv:2206.05895}, 2022.

\bibitem[ZTC22]{zhang2022gddim}
Qinsheng Zhang, Molei Tao, and Yongxin Chen.
\newblock gddim: Generalized denoising diffusion implicit models.
\newblock {\em arXiv preprint arXiv:2206.05564}, 2022.

\end{thebibliography}
